\documentclass[11pt]{article}
\usepackage{amsmath}
\usepackage{amssymb}
\usepackage{geometry}
\usepackage{bm}
\usepackage{algorithm}
\usepackage{algpseudocode}
\usepackage{graphicx}
\usepackage{caption}
\usepackage{subcaption}
\usepackage{multirow}
\usepackage{color}
\usepackage[inline, shortlabels]{enumitem}
\newcommand{\E}{\mathrm{E}}
\newtheorem{lemma}{Lemma}
\newtheorem{proposition}{Proposition}
\newtheorem{theorem}{Theorem}
\newtheorem{proof}{Proof}

\title{Multivariate Gaussian Network Structure Learning}

\author{Xingqi Du \\ 
Department of Statistics, North Carolina State University, \\ 5109 SAS Hall, Campus Box 8203, Raleigh, North Carolina 27695, USA \\ {xdu8@ncsu.edu} \\ \\ 
 Subhashis Ghosal \\
Department of Statistics, North Carolina State University, \\ 5109 SAS Hall, Campus Box 8203, Raleigh, North Carolina 27695, USA \\ {sghosal@ncsu.edu} }

\begin{document}

\maketitle

\section{Abstract}
We consider a graphical model where a multivariate normal vector is associated with each node of the underlying graph and estimate the graphical structure. We minimize a loss function obtained by regressing the vector at each node on those at the remaining ones under a group penalty. We show that the proposed estimator can be computed by a fast convex optimization algorithm. We show that as the sample size increases, the estimated regression coefficients and the correct graphical structure are correctly estimated with probability tending to one. By extensive simulations, we show the superiority of the proposed method over comparable procedures. We apply the technique on two real datasets. The first one is to identify gene and protein networks showing up in cancer cell lines, and the second one is to reveal the connections among different industries in the US.

\section{Introduction}
\label{sec:introduction}

Finding structural relations in a network of random variables $(X_i: i\in V)$ is a problem of significant interest in  modern statistics. The intrinsic dependence between variables in a network is appropriately described by a graphical model, where two nodes $i,j\in V$ are connected by an edge if and only if the two corresponding variables $X_i$ and $X_j$ are conditionally dependent given all other variables. If the joint distribution of all variables is multivariate normal with precision matrix $\Omega=(\!(\omega_{ij})\!)$, the conditional independence between the variable located at node $i$ and that located at node $j$ is equivalent of having zero at the $(i,j)$th entry of $\Omega$. In a relatively large network of variables, generally conditional independence is abundant, meaning that in the corresponding graph edges are sparsely present. Thus in a Gaussian graphical model, the structural relation can be learned from a sparse estimate of $\Omega$, which can be naturally obtained by regularization method with a lasso-type penalty. Friedman et al. \cite{glasso} and Banerjee et al.  \cite{Banerjee} proposed the graphical lasso (\texttt{glasso}) estimator by minimizing the sum of the negative log-likelihood and the $\ell_1$-norm of $\Omega$, and its convergence property was studied by Rothman et al. \cite{Rothman}. A closely related method was proposed by Yuan \& Lin \cite{Yuanlin}. An alternative to the graphical lasso is an approach based on regression of each variable on others, since $\omega_{ij}$ is zero if and only if the regression coefficient $\beta_{ij}$ of $X_j$ in regressing $X_i$ on other variables is zero. Equivalently this can be described as using a pseudo-likelihood obtained by multiplying one-dimensional conditional densities of $X_i$ given $(X_j, j\ne i)$ for all $i\in V$ instead of using the actual likelihood obtained from joint normality of $(X_i, i\in V)$. The approach is better scalable with dimension since the optimization problem is split into several optimization problems in lower dimensions. The approach was pioneered by Meinshausen \& B{\"u}hlmann \cite{NSlasso}, who imposed a lasso-type penalty on each regression problem to obtain sparse estimates of the regression coefficients, and showed that the correct edges are selected with probability tending to one. However, a major drawback of their approach is that the estimator of $\beta_{ij}$ and that of $\beta_{ji}$ may not be simultaneously zero (or non-zero), and hence may lead to logical inconsistency while selecting edges based on the estimated values. Peng et al. \cite{space} proposed the Sparse PArtial Correlation Estimation (\texttt{space}) by taking symmetry of the precision matrix into account. The method is shown to lead to convergence and correct edge selection with high probability, but it may be computationally challenging.  A weighted version of \texttt{space} was considered by Khare et al. \cite{concord}, who showed that a specific choice of weights guarantees convergence of the iterative algorithm due to the convexity of the objective funtion in its arguments. Khare et al. \cite{concord} named their estimator the CONvex CORrelation selection methoD (\texttt{concord}), and proved that the estimator inherits the theoretical convergence properties of \texttt{space}. By extensive simulation and numerical illustrations, they showed that \texttt{concord} has good accuracy for reasonable sample sizes and can be computed very efficiently. 

However, in many situations, such as if multiple characteristics are measured, the variables $X_i$ at different nodes $i\in V$ may be multivariate. The methods described above apply only in the context when all variables are univariate. Even if the above methods are applied by treating each component of these variables  as separate one-dimensional variables, ignoring their group structure may be undesirable, since all component variables refer to the same subject. For example, we may be interested in the connections among different industries in the US, and may like to see if the GDP of one industry has some effect on that of other industries. The data is available for 8 regions, and we want to take regions into consideration, since significant difference in relations may exist because of regional characteristics, which are not possible to capture using only national data. It seems that the only paper which addresses multi-dimensional variables in a graphical model context is Kolar et al.  \cite{multi-attribute}, who pursued a likelihood based approach.
In this article, we propose a method based on a pseudo-likelihood obtained from multivariate regression on other variables. We formulate a multivariate analog of \texttt{concord}, to be called \texttt{mconcord}, because of the computational advantages of \texttt{concord} in univariate situations. Our regression based approach appears to be more scalable than the likelihood based approach of Kolar et al.  \cite{multi-attribute}. Moreover, we provide theoretical justification by studying large sample convergence properties of our proposed method, while such properties have not been established for the procedure introduced by Kolar et al. \cite{multi-attribute}.

The paper is organized as follows. Section~\ref{sec:method}  introduces the \texttt{mconcord} method and describes its computational algorithm. Asymptotic properties of \texttt{mconcord} are presented in Section~\ref{sec:asymptotics}. Section~\ref{sec:simulation} illustrates the performance of \texttt{mconcord}, compared with other methods mentioned above. In Section~\ref{sec:application}, the proposed method is applied to two real data sets on gene/protein profiles and GDP respectively. Proofs are presented in Section~\ref{sec:proofs} and in the appendix.

\section{Method description}
\label{sec:method}

\subsection{Model and estimation procedure}

Consider a graph with $p$ nodes, where at the $i$th node there is an associated  $K_i$-dimensional random variable ${Y}_i=(Y_{i1},\dots,Y_{iK_i})^T$, $i=1,\ldots,p$. Let ${Y}=({Y}_1^T,\ldots,{Y}_p^T)^T$. Assume that ${Y}$ has multivariate normal distribution with zero mean and covariance matrix ${\Sigma}=(\!(\sigma_{ijkl})\!)$, where $\sigma_{ijkl}=\mathrm{cov}(Y_{ik},Y_{jl})$, $k=1,\ldots,K_i$, $l=1,\ldots,K_j$, $i,j=1,\ldots,p$. Let the precision matrix ${\Sigma}^{-1}$ be denoted by  ${\Omega}=(\!(\omega_{ijkl})\!)$, which can also be written as a block-matrix $(\!({\Omega}_{ij})\!)$. The primary interest is in the graph which describes the conditional dependence (or independence) between ${Y}_i$ and ${Y}_j$ given the remaining variables. We are typically interested in the situation where $p$ is relatively large and the graph is sparse, that is, most pairs ${Y}_i$ and ${Y}_j$, $i\ne j$, $i,j=1,\ldots,p$, are conditionally independent given all other variables. When $Y_i$ and $Y_j$ are conditionally independent given other variables, there will be no edge connecting $i$ and $j$ in the underlying graph; otherwise there will be an edge. Under the assumed multivariate normality of ${Y}$, it follows that there is an edge between $i$ and $j$ if and only if ${\Omega}_{ij}$ is a non-zero matrix. Therefore the problem of identifying the underlying graphical structure reduces to estimating the matrix ${\Omega}$ under the sparsity constraint that most off-diagonal blocks ${\Omega}_{ij}$ in the grand precision matrix ${\Omega}$ are zero. 

Suppose that we observe $n$ independent and identically distributed (i.i.d.) samples from the graphical model, which are collectively denoted by $\bm{Y}$, while $\bm{Y}_i$ stands for the sample of $n$ many $K_i$-variate observations at node $i$ and $\bm{Y}_{ik}$ stands for the vector of observations of the $k$th component at node $i$, $k=1,\ldots,K_i$, $i=1,\ldots,p$. Following the estimation strategies used in univariate Gaussian graphical models, we may propose a sparse estimator for ${\Omega}$ by minimizing a loss function obtained from the conditional densities of ${Y}_i$ given ${Y}_j$, $j\ne i$, for each $i$ and a penalty term. However, since sparsity refers to off-diagonal blocks rather than individual elements, the lasso-type penalty used in univariate methods like \texttt{space} or \texttt{concord} should be replaced by a group-lasso type penalty, involving the sum of the Frobenius-norms of each off-diagonal block ${\Omega}_{ij}$. A multivariate analog of the loss used in a weighted version of \texttt{space} is given by 
\begin{equation}
L_n(\omega,\sigma,\bm{Y})=\frac{1}{2}\sum_{i=1}^p\sum_{k=1}^{K_i}\Big(-\log\sigma^{ik}+\frac{w_{ik}}{n}\big\lVert\bm{Y}_{ik}+\sum_{j\neq i}\sum_{l=1}^{K_j}\frac{\omega_{ijkl}}{\sigma^{ik}}\bm{Y}_{jl}\big\rVert_2^2\Big),
\end{equation}
where $\sigma^{ik}=\omega_{iikk}$, $\bm{w}=(w_{11},\dots,w_{pK_p})$ are nonnegative weights and $\omega_{ijkl}=\omega_{jilk}$ due to the symmetry of precision matrix. Writing the quadratic term in the above expression as 
$$
w_{ik}\big\lVert\bm{Y}_{ik}+\sum_{j\neq i}\sum_{l=1}^{K_j}\frac{\omega_{ijkl}}{\sigma^{ik}}\bm{Y}_{jl}\big\rVert_2^2  = \frac{w_{ik}}{(\sigma^{ik})^2}\big\lVert\sigma^{ik}\bm{Y}_{ik}+\sum_{j\neq i}\sum_{l=1}^{K_j}{\omega_{ijkl}}\bm{Y}_{jl}\big\rVert_2^2,
$$
and, as in \texttt{concord} choosing $w_{ik}=(\sigma^{ik})^2$ to make the optimization problem convex in the arguments, we can write the quadratic term in the loss function as $\lVert\sigma^{ik}\bm{Y}_{ik}+\sum_{j\neq i}\sum_{l=1}^{K_j}\omega_{ijkl}\bm{Y}_{jl}\rVert_2^2$. 
Applying the group penalty we finally arrive at the objective function 
\begin{equation}
\label{eq:objective function}
 \frac{1}{2}\sum_{i=1}^p\sum_{k=1}^{K_i}\Big(-\log\sigma^{ii}+\frac{1}{n}\big\lVert\sigma^{ik}\bm{Y}_{ik}+\sum_{j\neq i}\sum_{l=1}^{K_j}\omega_{ijkl}\bm{Y}_{jl}\big\rVert_2^2\Big)
 +\lambda\sum_{i<j}\Big(\sum_{k=1}^{K_i}\sum_{l=1}^{K_j}\omega_{ijkl}^2\Big)^{1/2}.
\end{equation}

\subsection{Algorithm}

To obtain a minimizer of \eqref{eq:objective function}, we periodically minimize it with respect to the arguments of ${\Omega}_{ij}$, $i\ne j$, $i,j=1,\ldots,p$. 
For each fixed $(i,j)$, $i\ne j$, suppressing the terms not involving any element of ${\Omega}_{ij}$, we may write the objective function as 
$$
\frac{1}{2n}\Big(\sum_{k=1}^{K_i}\lVert\sigma^{ik}\bm{Y}_{ik}+\sum_{j'\neq i}\sum_{l=1}^{K_{j'}}\omega_{ij'kl}\bm{Y}_{j'l}\rVert_2^2\\
+\sum_{l=1}^{K_j}\lVert\sigma^{jl}\bm{Y}_{jl}+\sum_{i'\neq j}\sum_{k=1}^{K_{i'}}\omega_{i'jlk}\bm{Y}_{ik}\rVert_2^2\Big)+\lambda\lVert\omega_{ij}\rVert_2, 
$$ 
where ${\omega}_{ij}=\mathrm{vec}({\Omega}_{ij})$. 
Without loss of generality, we assume $i<j$ and rewrite the expression as
\begin{eqnarray*}
\lefteqn
 {\frac{1}{2n}\Big(\sum_{k=1}^{K_i}\lVert\sigma^{ik}\bm{Y}_{ik}+\bm{B}_{1jk}\omega_{ij}+\sum_{j'> i,j'\neq j}\bm{B}_{1j'k}\omega_{ij'}+\sum_{j'< i}\bm{B}_{2j'k}\omega_{ij'}\rVert_2^2}\\
&&+\sum_{l=1}^{K_j}\lVert\sigma^{jl}\bm{Y}_{jl}+ \bm{B}_{2il}\omega_{ij}+ \sum_{i'>j}\bm{B}_{1i'l}\omega_{i'j}+\sum_{i'<j,i'\neq i}\bm{B}_{2i'l}\omega_{i'j}\rVert_2^2\Big)
+\lambda\lVert\omega_{ij}\rVert_2,
\end{eqnarray*}
 where $\bm{B}_{1jk}$ and $\bm{B}_{2il}$ are $n\times K_iK_j$ matrices specified as follows: $((k-1)K_j+1,\dots,kK_j)$th columns of $\bm{B}_{1jk}$ are $\bm{Y}_{j}$, the $(l,K_j+l,\dots,(K_i - 1)K_j+l)$th columns of $\bm{B}_{2il}$ are $\bm{Y}_{i}$, and other columns are zero. This leads to the following algorithm.

\textit{Algorithm: }

\textit{Initialization:} For $k=1,\ldots,K_i$, and $i=1,\ldots,p$, set the initial values $\hat{\sigma}^{ik}=1/\widehat{\text{var}}({Y}_{ik})$ and $\hat{\omega}_{ij}={0}$. 

\textit{Iteration:} 
For all $1\leq i\leq p$ and $1\leq k\leq K_i$, repeat the following steps until certain convergence criterion is satisfied: 
\begin{description} 
\item \textit{Step 1: }Calculate the vectors of errors for $\omega_{ij}$:
$$
\begin{aligned}
& \bm{r}_{ijk} = \hat{\sigma}^{ik}\bm{Y}_{ik}+\sum_{j'< i}\bm{B}_{2j'k}\hat{\omega}_{j'i}+\sum_{j'> i,j'\neq j}\bm{B}_{1j'k}\hat{\omega}_{ij'},\\
& \bm{r}_{jil} = \hat{\sigma}^{jl}\bm{Y}_{jl}+\sum_{i'> j}\bm{B}_{1i'l}\hat{\omega}_{ji'}+\sum_{i'< j,i'\neq i}\bm{B}_{2i'l}\hat{\omega}_{i'j}.
\end{aligned}
$$
\item \textit{Step 2: }Regress the errors on the specified variables to obtain 
\begin{eqnarray*}
	\hat{\omega}_{ij} &=& \arg\min\Big[\frac{1}{2n}\Big\{\omega_{ij}^T\big(\sum_{k=1}^{K_i}\bm{B}_{1jk}^T\bm{B}_{1jk}+\sum_{l=1}^{K_j}\bm{B}_{2il}^T\bm{B}_{2il}\big)\omega_{ij}\\ && \quad +2\big(\sum_{k=1}^{K_i}\bm{r}_{ijk}^T\bm{B}_{1jk}+\sum_{l=1}^{K_j}\bm{r}_{jil}^T\bm{B}_{2il}\big)\omega_{ij}\Big\}+\lambda\lVert\omega_{ij}\rVert_2\Big], 
\end{eqnarray*}
by the proximal gradient algorithm described as follows: 

	Given $\omega_{ij}^{(t)}$, $\bm{r}_{ijk}^{(t+1)}$ and $\bm{r}_{jil}^{(t+1)}$, compute 
	\begin{eqnarray*} 
	f(\omega_{ij}^{(t)})&=& \frac{1}{2n}\Big[\omega_{ij}^{(t)T}\big(\sum_{k=1}^{K_i}\bm{B}_{1jk}^T\bm{B}_{1jk}+\sum_{l=1}^{K_j}\bm{B}_{2il}^T\bm{B}_{2il}\big)\omega_{ij}^{(t)}\\
	&&\quad+2\big(\sum_{k=1}^{K_i}\bm r_{ijk}^{(t+1)T}\bm{B}_{1jk}+\sum_{l=1}^{K_j}\bm r_{jil}^{(t+1)T}\bm{B}_{2il}\big)\omega_{ij}^{(t)}\Big]\\
	{g} &=& \frac{1}{n}\bigg(\sum_{k=1}^{K_i}\Big(\bm{B}_{jk}^T\bm{B}_{jk}\omega_{ij}^{(t)}+\bm r_{ijk}^{(t+1)T}\bm{B}_{jk}\Big)+\sum_{l=1}^{K_j}\Big(\bm{B}_{il}^T\bm{B}_{il}\omega_{ij}^{(t)}+\bm r_{jil}^{(t+1)T}\bm{B}_{il}\Big)\bigg)
	\end{eqnarray*}
	Set $s\leftarrow 1$ and repeat 
		\begin{itemize} 
		\item ${z}_{ij} \leftarrow \omega_{ij}^{(t)} -  s{g}$, 
		\item if $\lVert \bm{z}_{ij}\rVert_2\geq\lambda^2 s^2$, set 
		 $ \omega_{ij}^{(t+1)}\leftarrow \Big(1-\frac{\lambda s}{\lVert {z}_{ij}\rVert_2}\Big){z}_{ij}$; else 
		set $\omega_{ij}^{(t+1)}\leftarrow{0}$, 
		\item 
		replace $s$ by $s/2$, 
		\end{itemize}
		until $f(\omega_{ij}^{(t)})\leq f(\omega_{ij}^{(t+1)})+{g}^T(\omega_{ij}^{(t+1)}-\omega_{ij}^{(t)})+\frac{1}{2s}\lVert\omega_{ij}^{(t+1)}-\omega_{ij}^{(t)}\rVert_2^2$.

\item \textit{Step 3: }
For $1\leq i\leq p$ and $1\leq k\leq K_i$, update $\hat{\sigma}^{ik}$ to 
$$ \frac{-\bm{Y}_{ik}^T(\sum\limits_{j< i}\bm{B}_{2jk}{\hat{\omega}}_{ij}+\sum\limits_{j> i}\bm{B}_{1jk}{\hat{\omega}}_{ij})+\sqrt{\Big(\bm{Y}_{ik}^T(\sum\limits_{j< i}\bm{B}_{2jk}{\hat{\omega}}_{ij}+\sum\limits_{j> i}\bm{B}_{1jk}{\hat{\omega}}_{ij})\Big)^2+2n\bm{Y}_{ik}^T\bm{Y}_{ik}}}{2\bm{Y}_{ik}^T\bm{Y}_{ik}}.
$$
\end{description}

If the total number of variables at all nodes $\sum_{i=1}^p K_i$ is less than or equal to the available sample size $n$, then the objective function is strictly convex, there is a unique solution to the minimization problem \eqref{eq:objective function} and the iterative scheme converges to the global minimum (Tseng \cite{Tseng}). However, if $\sum_{i=1}^p K_i>n$, the objective function need not be strictly convex, and hence a unique minimum is not guaranteed. However, as in univariate \texttt{concord}, the algorithm converges to a global minimum. This follows by arguing as in the proof of Theorem~1 of Kolar et al. \cite{concord} after observing that the objective function of \texttt{mconcord} differs from that of \texttt{concord} only in two aspects --- the loss function does not involve off-diagonal entries of diagonal blocks, and the penalty function has grouping, neither of which affect the structure of the \texttt{concord} described by Equation (33) of Kolar et al. \cite{concord}.

\section{Large Sample Properties}
\label{sec:asymptotics}

In this section, we study large sample properties of the proposed \texttt{mconcord} method. As in the univariate \texttt{concord} method, we consider the estimator obtained from the minimization problem 
$$\frac{1}{2} \sum_{i=1}^p\sum_{k=1}^{K_i}\Big(-\log\hat{\sigma}^{ik}+\frac{w_{ik}}{n}\lVert\bm{Y}_{ik}+\lambda\sum_{j\neq i}\sum_{l=1}^{K_j}\frac{\omega_{ijkl}}{\hat{\sigma}^{ik}}\bm{Y}_{jl}\rVert_2^2\Big)+\lambda_n\sum_{i<j}\Big(\sum_{k=1}^{K_i}\sum_{l=1}^{K_j}\omega_{ijkl}^2\Big)^{1/2}$$
with a general weight $w_{ik}$ and a suitably consistent estimator $\hat{\sigma}^{ik}$ of $\sigma^{ik}$ plugged in for all $k=1,\ldots,K_i$, $i=1,\ldots,p$, and for some suitable sequence $\lambda_n$. Existence of such an estimator is also shown. 

Introduce the notation 
\begin{equation} 
\label{asy1}
	L({\omega},{\sigma},{Y})=\frac{1}{2}\sum_{i=1}^p\sum_{k=1}^{K_i}w_{ik}\Big(Y_{ik}+\sum_{j\neq i}\sum_{l=1}^{K_j}\frac{\omega_{ijkl}}{\sigma^{ik}}Y_{jl}\Big)^2,
\end{equation}
where ${\sigma}=(\sigma^{ik}:k=1,\ldots,K_i, i=1,\ldots,p)$ and $\omega=(\omega_{ijkl}:k=1,\ldots,K_i$, $l=1,\ldots,K_j$, $i,j=1,\ldots, p, i\ne j)$. Let $\bar{{\omega}}$ and $\bar{{\sigma}}$ respectively stand for true values of ${\Omega}$ and ${\sigma}$ respectively. All probability and expectation statements made below are understood under the distributions obtained from the true parameter values. Let $\bar{L}_{ijkl}'({\omega},{\sigma},{Y})=\E \Big(\frac{\partial}{\partial \omega_{ijkl}}{L}({\omega},{\sigma},{Y})\big|_{\omega=\bar{\omega},\sigma=\bar{\sigma}}\big)$ and $\bar{L}''_{ijkl,i'j'k'l'}(\bar{\omega},\bar{\sigma})=\text{E}\bigg(\frac{\partial^2}{\partial\omega_{ijkl}\partial\omega_{i'j'k'l'}}L(\omega,\sigma,Y)|_{\omega=\bar{\omega},\sigma=\bar{\sigma}}\bigg)$ be the expected first and second order partial derivatives of ${L}$ at the true parameter respectively. Also let $\bar{L}_{ijkl,S}''$ stand for the row vector $(\bar{L}_{ijkl,i'j'k'l'}'': (i'j'k'l')\in S)$ and $\bar{L}_{S,S}''$ for the matrix $(\!(\bar{L}_{ijkl,i'j'k'l'}'': {ijkl,i'j'k'l'\in S} )\!)$, where $S\subset T :=\{(i,j,k,l):1\leq i\neq j\leq p,1\leq k\leq K_i,1\leq l\leq K_j\}$. 
Note that $\bar{L}''_{ijkl,i'j'k'l'}(\bar{\omega},\bar{\sigma})=\text{E}[Y_{jl}Y_{j'l'}+Y_{ik}Y_{i'lk}]=\sigma_{jl,j'l'}+\sigma_{ik,i'k'}$. 

Let $\mathcal{A}_0=\{(i,j):\exists k\in\{1,\dots,K_i\}, \exists l\in\{1,\dots,K_j\},\bar{\omega}_{ijkl}\neq 0\}$, and $q_n=|\mathcal{A}_0|$. We further define that $\mathcal{A}=\{(i,j,k,l):(i,j)\in\mathcal{A}_0,1\leq k\leq K_i,1\leq l\leq K_j\}$, and thus there are $\sum_{(i,j)\in\mathcal{A}_0}K_iK_j$ elements in $\mathcal{A}$. Let $K_{\max}=\max\{K_i: i=1,\ldots,p\}$. The following assumptions will be made throughout. 

\begin{itemize}
\item [(C0)] The weights satisfy $0<w_0\leq \min(w_{ik})\leq \max(w_{ik})\leq w_{\infty}<\infty$ and $K_{\max}$ and $p$ grow at most like a power of $n$.  

\item [(C1)] There exist constants $0<\Lambda_{\min}\le \Lambda_{\max}$ depending on the true parameter value such that the minimum and maximum eigenvalues of the true covariance $\bar{\bm\Sigma}$ satisfies  $0<\Lambda_{\min}\leq\lambda_{\min}(\bar{\bm\Sigma})\leq\lambda_{\max}(\bar{\bm\Sigma})\leq\Lambda_{\max}<\infty$. 

\item [(C2)] There exists a constant $\delta<1$ such that for all $(i,j,k,l)\not\in \mathcal{A}$, 
$
|\bar{L}''_{ijkl,\mathcal{A}}(\bar{\omega},\bar{{\sigma}})[\bar{L}''_{\mathcal{A},\mathcal{A}}(\bar{\omega},\bar{{\sigma}})]^{-1}{M}|\leq\delta,
$
where ${M}$ is a column-vector with elements ${\bar{\omega}_{ijkl}}/{\sqrt{\sum\limits_{k',l'}\bar{\omega}_{ijk'l'}^2}}$, $(i,j,k,l)\in \mathcal{A}$. 

\item [(C3)] There is an estimator $\hat{\sigma}^{ik}$ of ${\sigma}^{ik}$, $k=1,\ldots,K_i$ satisfying  ${\max}\{\lvert\hat{\sigma}^{ik}-\bar{\sigma}^{ik}\rvert : {1\leq i\leq p,1\leq k\leq K_i}\}\leq C_n\sqrt{({\log n})/{n}}$ for every $C_n\to\infty$ with probability tending to 1.
\end{itemize}

The following result concludes that Condition C3 holds if the total dimension is less than a fraction of the sample size. 

\begin{proposition}
\label{pro:sigma}
Suppose that $\sum_{i=1}^p K_i\le \beta n$ for some $0<\beta<1$. Let $\bm{e}_{ik}$ stand for the vector of regression residuals of $Y_{ik}$ on $\{Y_{il}: l\ne k\}$. Then the estimator   $\hat{\sigma}^{ik}=1/\hat{\sigma}_{ik,-ik}$, where
$\hat{\sigma}_{ik,-ik}=({n-\sum_{j\neq i}K_j})^{-1}\bm{e}_{ik}^T\bm{e}_{ik}$, satisfies Condition C3. 
\end{proposition}

We adapt the approach in Peng et al. \cite{space} to the multivariate Gaussian setting. The approach consists of first showing that if the estimator is restricted to the correct model, then it converges to the true parameter at a certain rate as the sample size increases to infinity. The next step consists of showing that with high probability no edge is falsely selected. These two conclusions combined yield the result. 

\begin{theorem} 
\label{restricted problem}
Let $K_{\max}^2q_n= o(\sqrt{n/\log n})$, $\lambda_n\sqrt{n/\log n}\rightarrow\infty$ and $K_{\max}\sqrt{q_n}\lambda_n=o(1)$ as $n\rightarrow\infty$. Then the following events hold with probability tending to $1$:
\begin{enumerate} 
\item [(i)] there exists a solution $\hat{\omega}_{\mathcal{A}}^{\lambda_n}=\hat{\omega}_{\mathcal{A}}^{\lambda_n}(\hat{{\sigma}})$ of the restricted problem
\begin{equation} 
\label{thm1}
	\arg\min_{\omega:{\omega}_{\mathcal{A}^c}=0}L_n(\omega,\hat{{\sigma}},\bm{Y})+\lambda_n\sum_{i<j}\lVert\omega_{ij}\rVert_2.
\end{equation}
\item [(ii)] (estimation consistency) for any sequence $C_n\rightarrow\infty$, any solution $\hat{\omega}_{\mathcal{A}}^{\lambda_n}$ of the restricted problem \eqref{thm1} satisfies $\lVert\hat{\omega}_{\mathcal{A}}^{\lambda_n}-\bar{\omega}_\mathcal{A}\rVert_2\leq C_nK_{\max}\sqrt{q_n}\lambda_n.$
\end{enumerate}
\end{theorem}

\begin{theorem} 
\label{selection}
Suppose that $K_{\max}^2p=O(n^{\kappa})$ for some $\kappa\geq0$, $K_{\max}^2q_n= o(\sqrt{n/\log n})$, $K_{\max}\sqrt{q_n\log n/n} =o(\lambda_n)$, $\lambda_n\sqrt{n/\log n}\rightarrow\infty$ and $K_{\max}\sqrt{q_n}\lambda_n=o(1)$ as $n\rightarrow\infty$. Then with probability tending to $1$, the solution of \eqref{thm1} satisfies
$
\max \{|L'_{n,ijkl}(\hat{{\Omega}}^{\mathcal{A},\lambda_n},\hat{{\sigma}},\bm{Y})|: _{(i,j,k,l)\in \mathcal{A}^c}\}
<\lambda_n,
$
where $L'_{n,ijkl}=\partial L_n/\partial\omega_{ijkl}$. 
\end{theorem}

\begin{theorem} 
\label{consistency} 
Assume that the sequences $K_{\max}, p, q_n$ and $\lambda_n$ satisfy the conditions in Theorem~\ref{selection}. Then with probability tending to $1$, there exists a minimizer $\hat{\omega}^{\lambda_n}$ of $L_n(\omega,\hat{{\sigma}},\bm{Y})+\lambda_n\sum_{i<j}\lVert\omega_{ij}\rVert_2$ which satisfies 
\begin{enumerate} 
\item [(i)] (estimation consistency) for any sequence $C_n\rightarrow\infty$, 
$\lVert\hat{\omega}^{\lambda_n}-\bar{\omega}\rVert_2\leq C_n K_{\max}\sqrt{q_n}\lambda_n$, 
\item [(ii)] (selection consistency) if for some $C_n\rightarrow\infty$, $\lVert\bar{\omega}_{ij}\rVert_2>C_nK_{\max}\sqrt{q_n}\lambda_n$ whenever $\bar{\omega}_{ij}\neq0$, then $\hat{\mathcal{A}}=\mathcal{A}$, where $\hat{\mathcal{A}}=\{(i,j): \hat{{\omega}}_{ij}^{\lambda_n}\ne {0}\}$. 
\end{enumerate}
\end{theorem}

\section{Simulation}
\label{sec:simulation}
In this section, two simulation studies are conducted to examine the performance of \texttt{mconcord} and compare with \texttt{space}, \texttt{concord}, \texttt{glasso} and \texttt{multi}, the method of Kolar et al. \cite{multi-attribute} in regards of estimation accuracy and model selection. For \texttt{space}, \texttt{concord} and \texttt{glasso}, all components of each node are treated as separate univariate nodes, and we put an edge between two nodes as long as there is at least one non-zero entry in the corresponding submatrix. 

\subsection{Estimation Accuracy Comparison}
In the first study, we evaluate the performance of each method at a series of different values of the tuning parameter $\lambda$. Four random networks with $p=30$ (44\% density), $p=50$ (21\% density), $p=100$ (6\% density), $p=200$ (2\% density) and $p=350$ (2\% density) nodes are generated, and each node has a $K$-dimensional Gaussian variable associate with it, $K=3, 5, 8$. Based on each network, we construct a $pK\times pK$ precision matrix, with non-zero blocks corresponding to edges in the network. Elements of diagonal blocks are set as random numbers from $[0.5, 1]$. If node $i$ and node $j$ $(i<j)$ are not connected, then the entire $(i,j)$th and $(j,i)$th blocks would take values zero. If node $i$ and node $j$ $(i<j)$ are connected, the $(i,j)$th block would have elements taking values in $(0, 0.05, -0.05, -0.2, 0.2)$ with equal probabilities so that both strong and weak signals are included. The $(j,i)$th block can be obtained by symmetry. Finally, we add $\rho I$ to the precision matrix to make it positive-definite, where $\rho$ is the absolute value of the smallest eigenvalue plus 0.5 and $I$ is the identity matrix. Using each precision matrix, we generate 50 independent datasets consisting of $n=50$ (for the $p=30$ and $p=50$ networks) and $n=100$ (for the $p=100$, $p=200$ and $p=350$ networks) i.i.d. samples. Results are given in Figure~1 to Figure~5. All figures show the number of correctly detected edges ($N_c$) versus the number of total detected edges ($N_t$), averaged across the 30 independent datasets. 

\begin{figure}[H]
	\label{fig1}
	\centering
	\begin{subfigure}[t]{0.3\textwidth}
		\includegraphics[width=\textwidth]{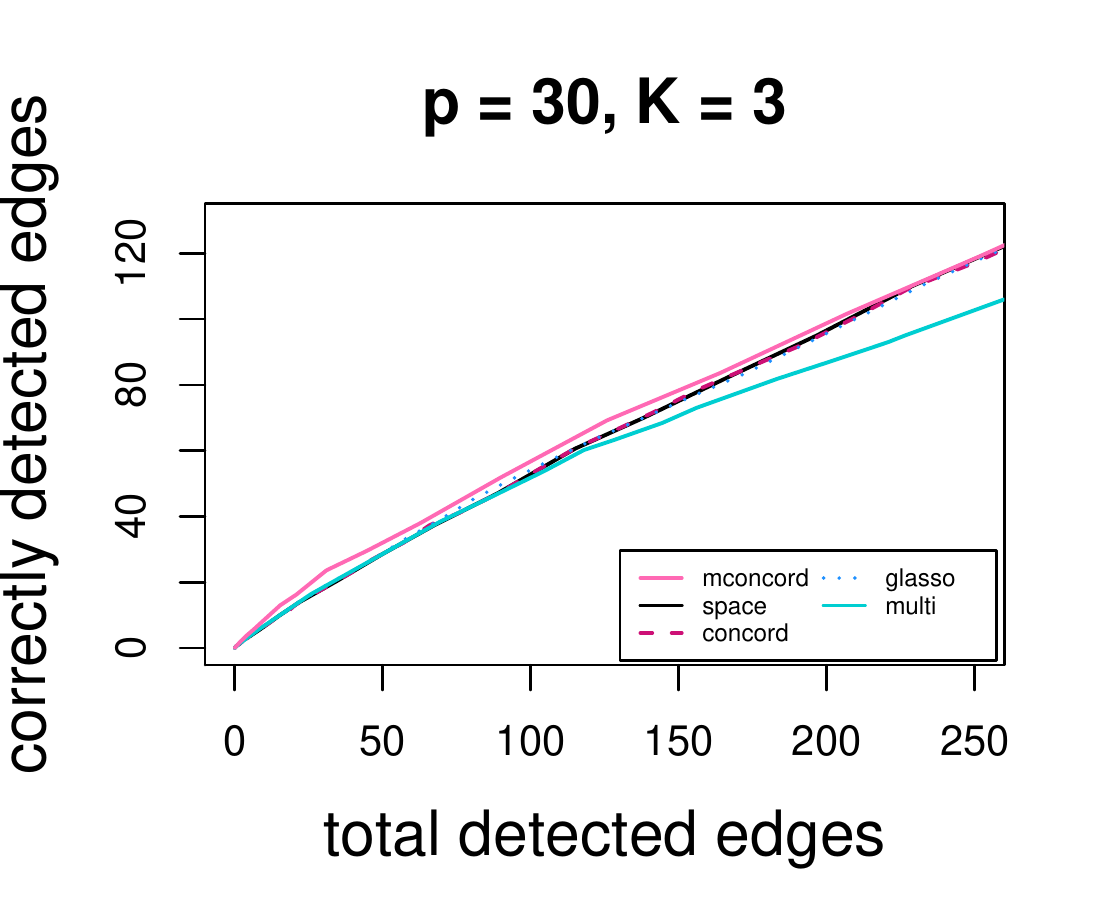}
		\caption{}
	\end{subfigure}
	\begin{subfigure}[t]{0.3\textwidth}
		\includegraphics[width=\textwidth]{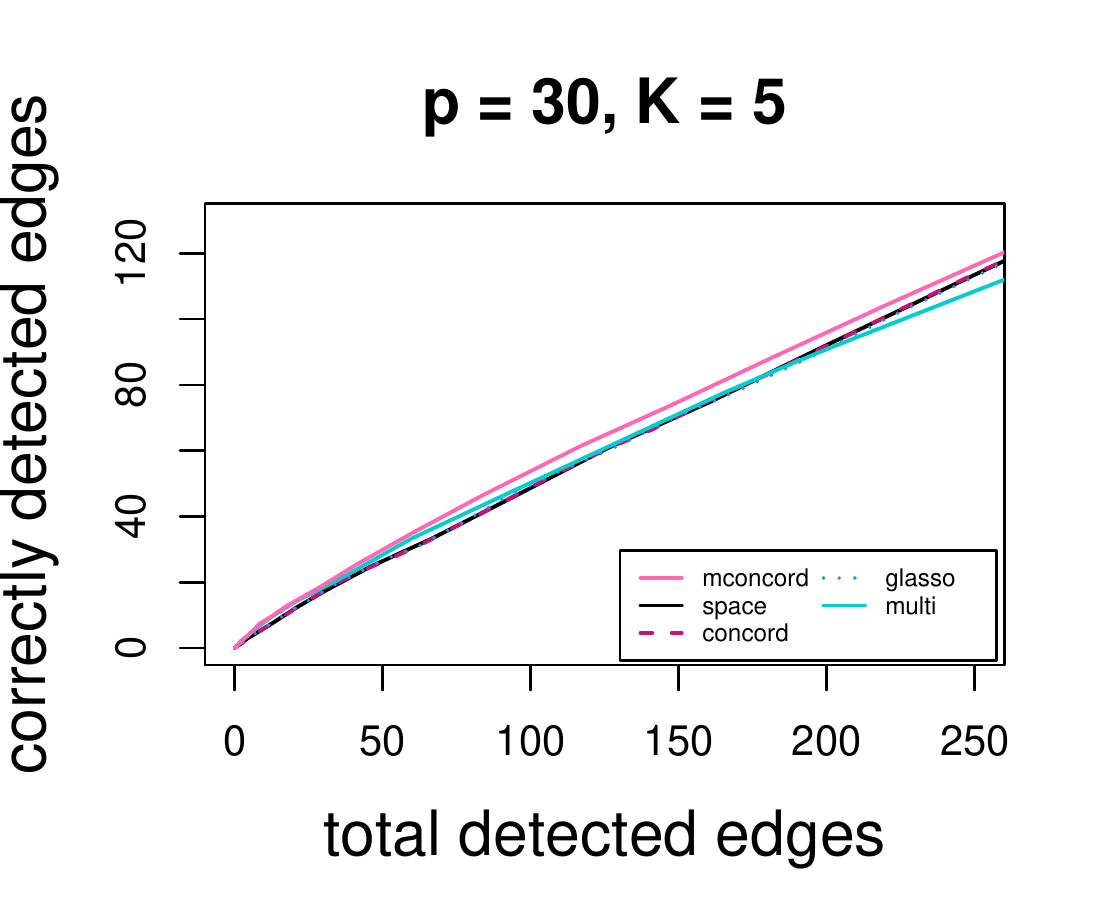}
		\caption{}
	\end{subfigure}
	\begin{subfigure}[t]{0.3\textwidth}
		\includegraphics[width=\textwidth]{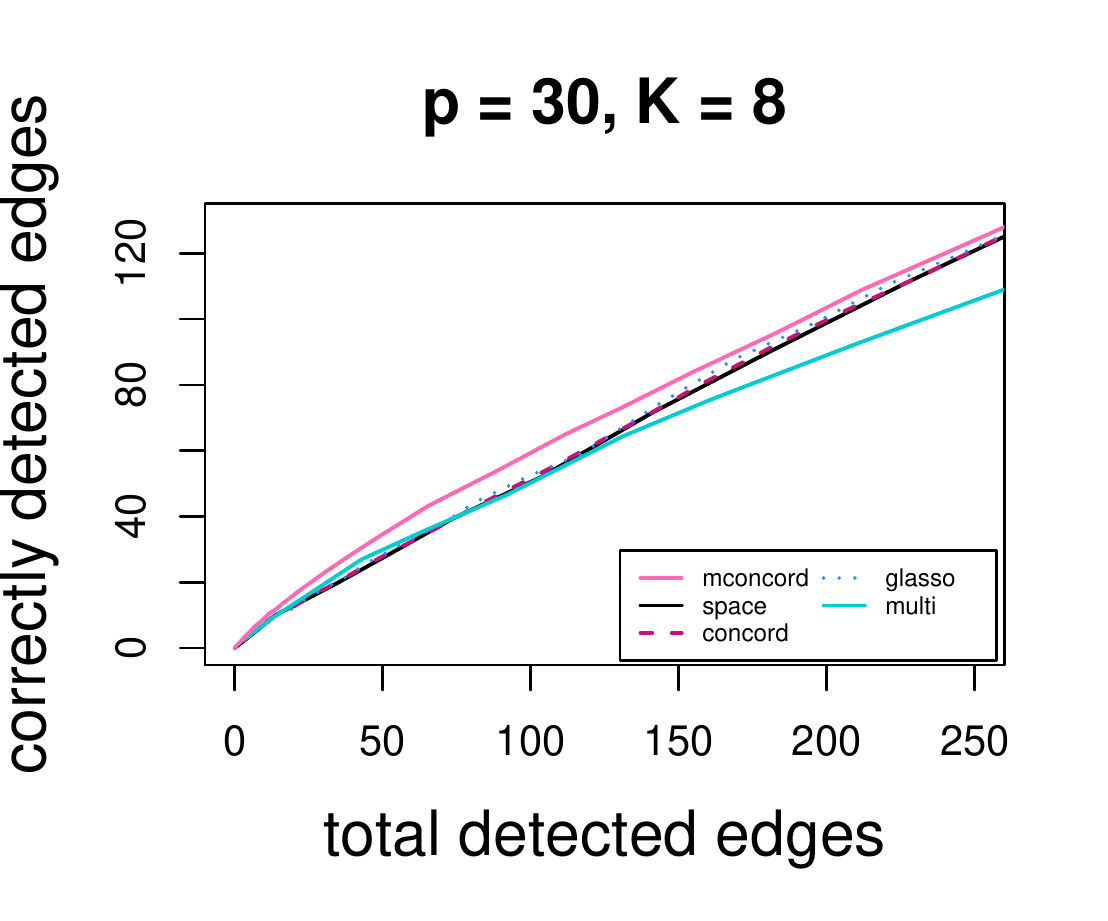}
		\caption{}
	\end{subfigure}
	\caption{Estimation accuracy comparison: total detected edges vs. correctly detected edges with 190 true edges (44\%): (a) $K=3$; (b) $K=5$; (c) $K=8$}
\end{figure}

\begin{figure}[H]
	\label{fig2}
	\centering
	\begin{subfigure}[t]{0.3\textwidth}
		\includegraphics[width=\textwidth]{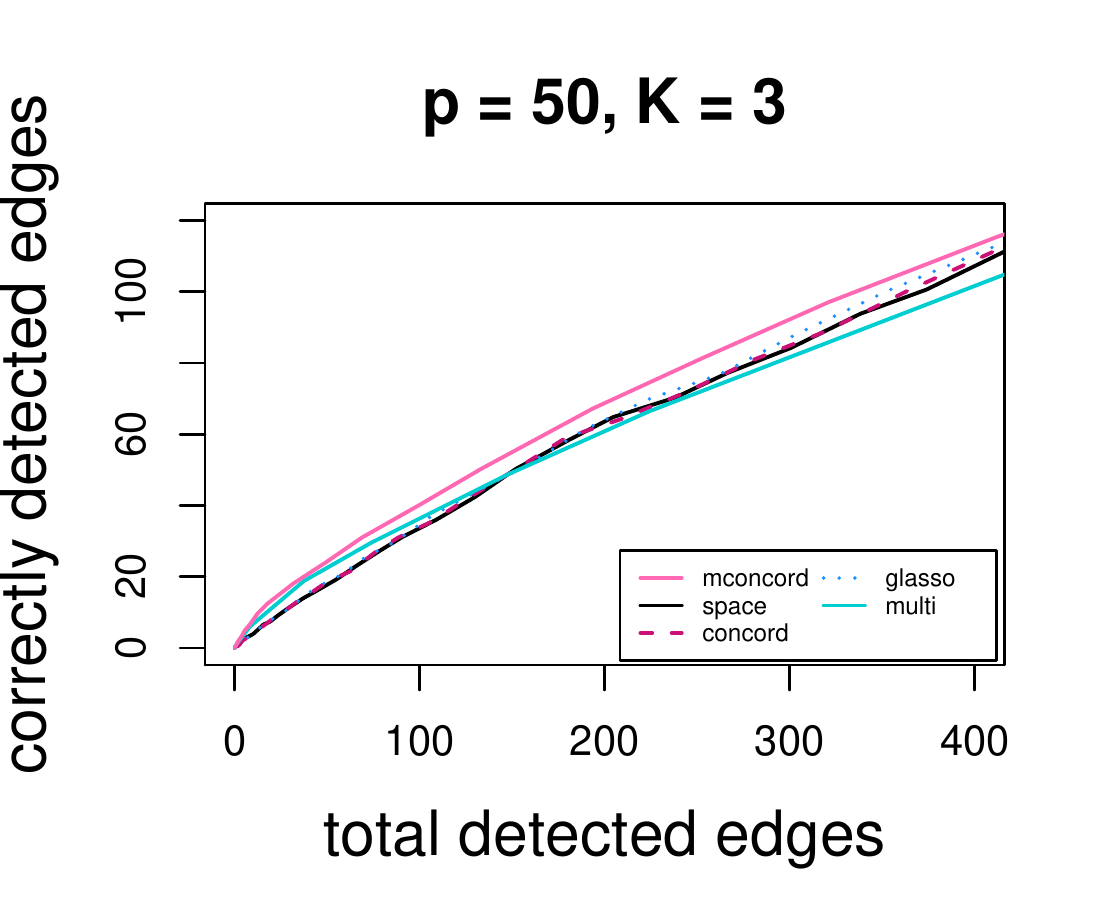}
		\caption{}
	\end{subfigure}
	\begin{subfigure}[t]{0.3\textwidth}
		\includegraphics[width=\textwidth]{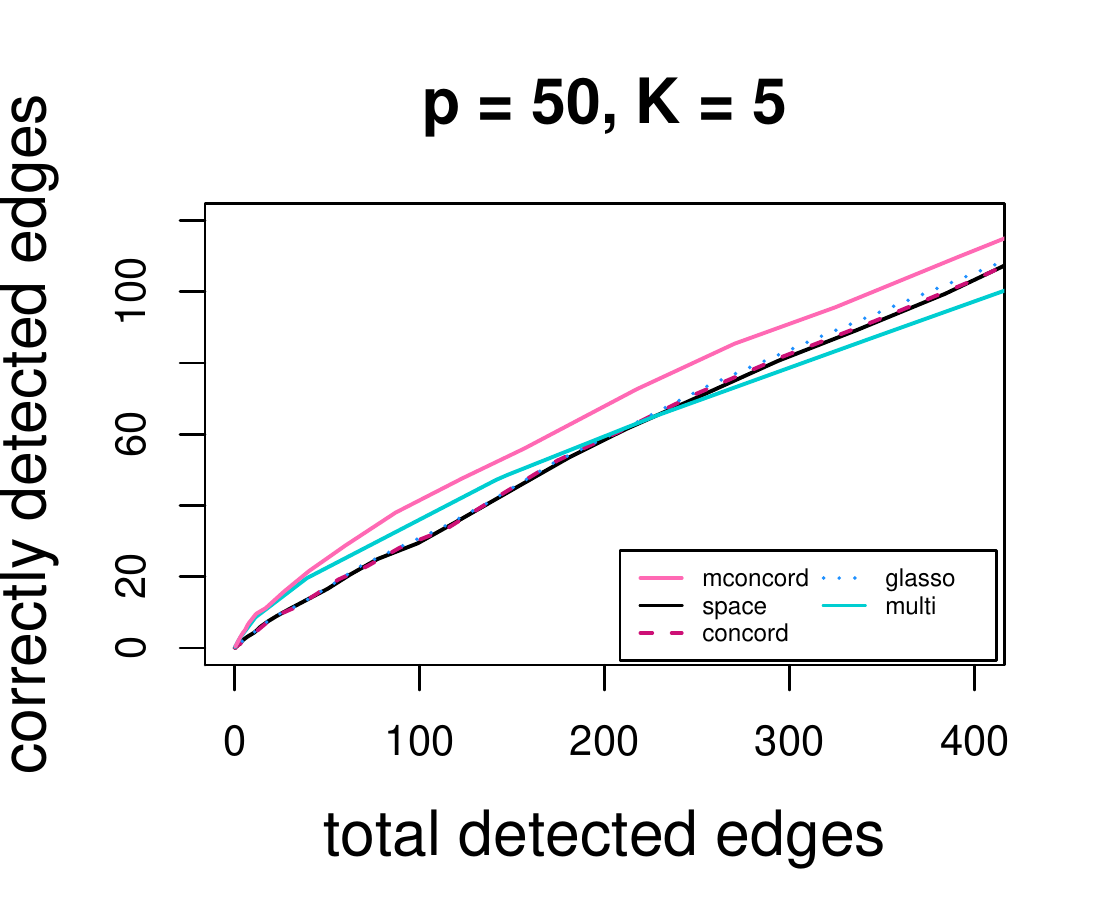}
		\caption{}
	\end{subfigure}
	\begin{subfigure}[t]{0.3\textwidth}
		\includegraphics[width=\textwidth]{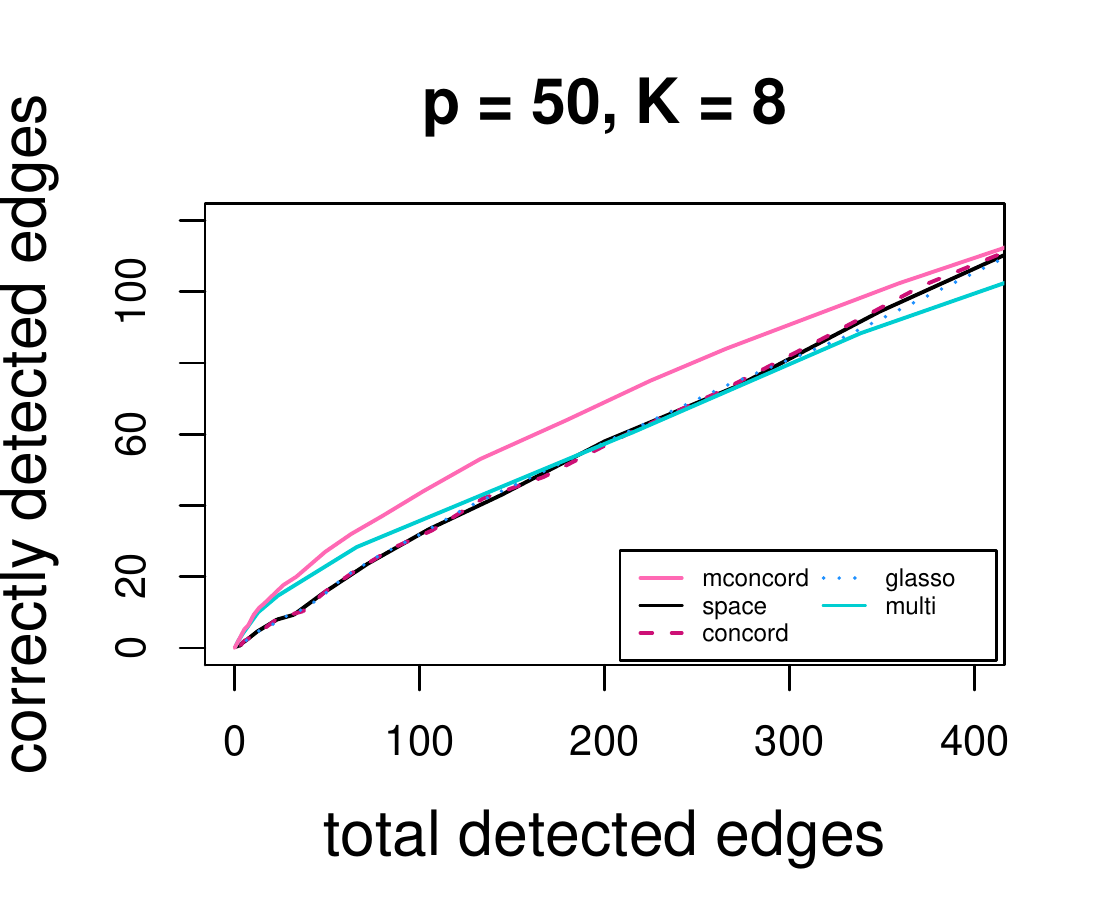}
		\caption{}
	\end{subfigure}
	\caption{Estimation accuracy comparison: total detected edges vs. correctly detected edges with 262 true edges (21\%): (a) $K=3$; (b) $K=5$; (c) $K=8$}
\end{figure}

\begin{figure}[H]
\label{fig3}
	\centering
	\begin{subfigure}[t]{0.3\textwidth}
		\includegraphics[width=\textwidth]{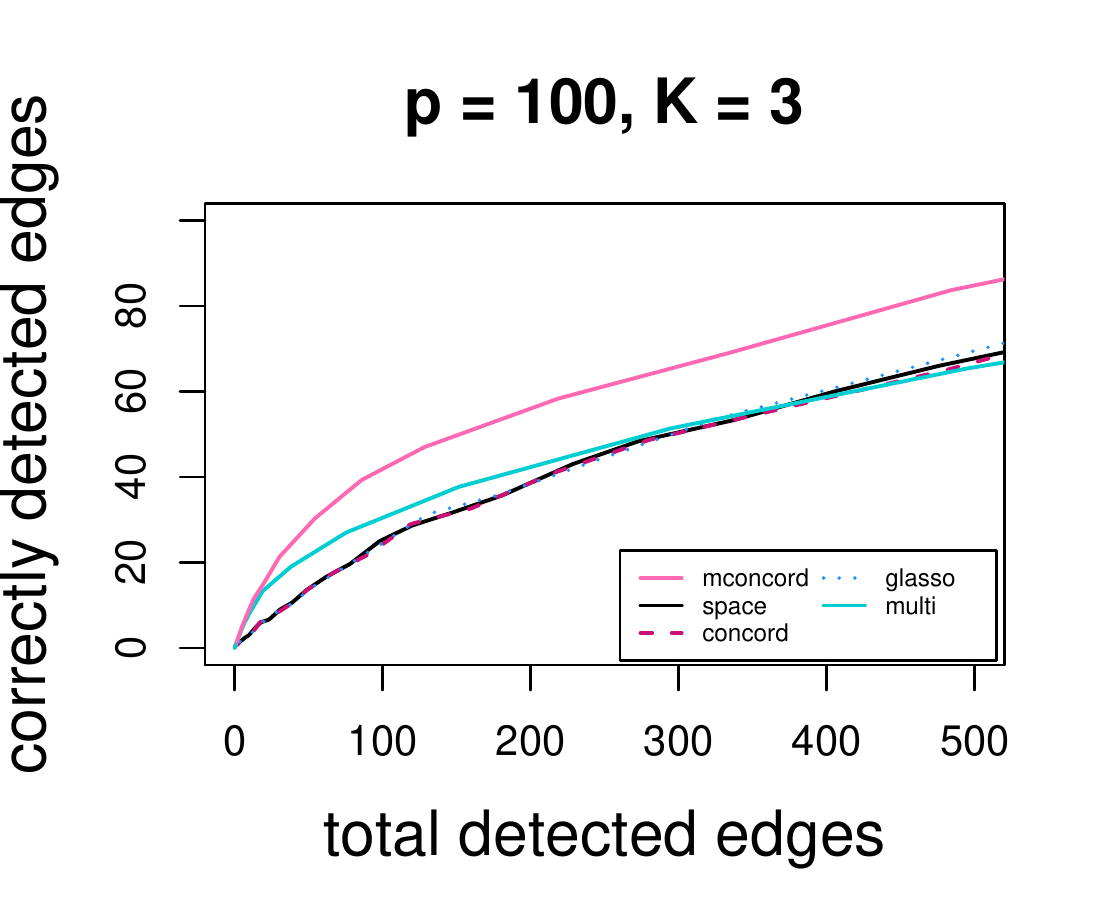}
		\caption{}
	\end{subfigure}
	\begin{subfigure}[t]{0.3\textwidth}
		\includegraphics[width=\textwidth]{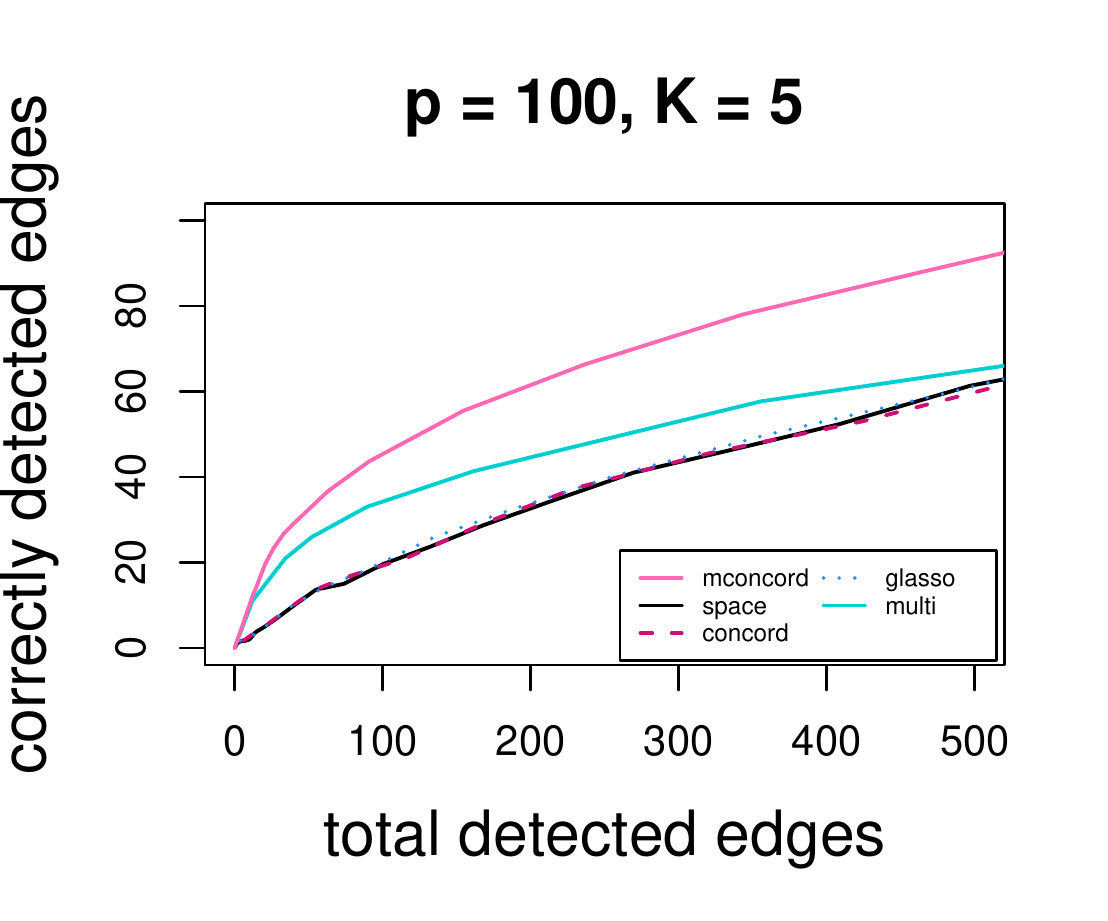}
		\caption{}
	\end{subfigure}
	\begin{subfigure}[t]{0.3\textwidth}
		\includegraphics[width=\textwidth]{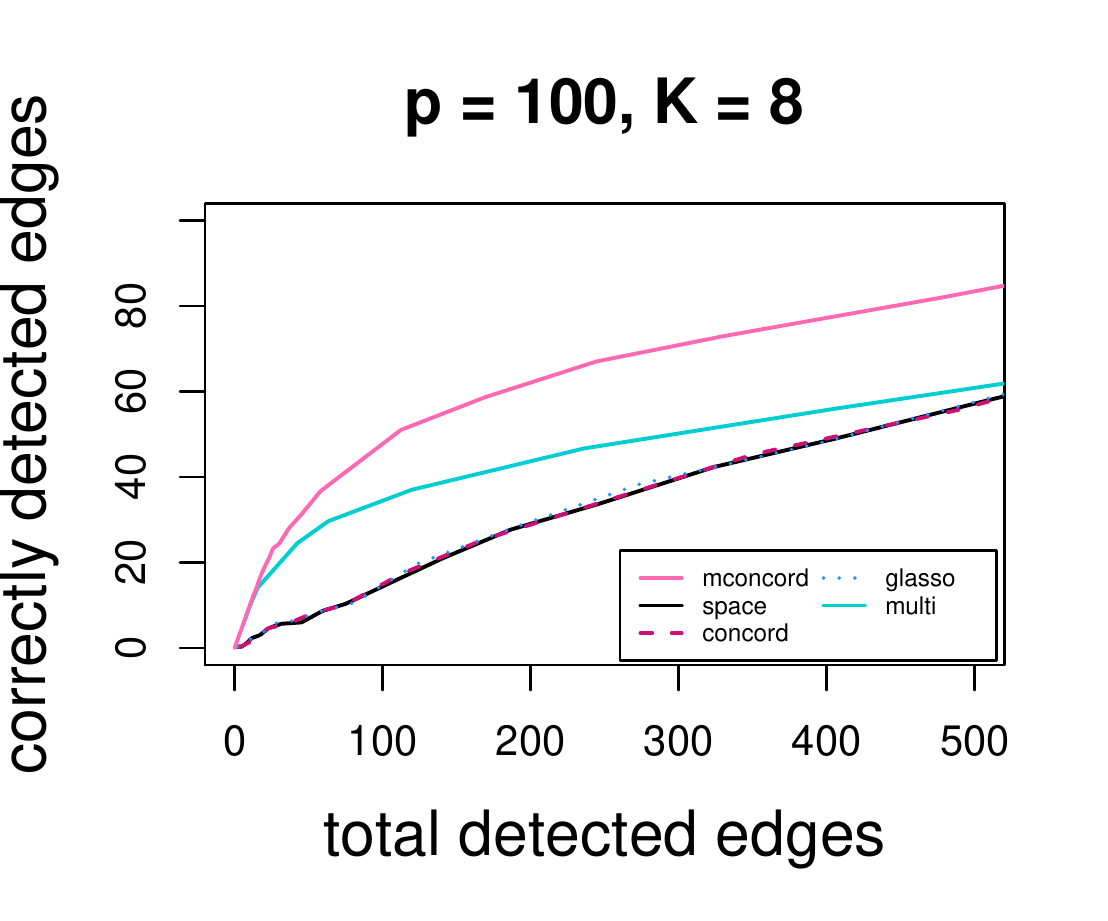}
		\caption{}
	\end{subfigure}
	\caption{Estimation accuracy comparison: total detected edges vs. correctly detected edges with 279 true edges (6\%): (a) $K=3$; (b) $K=5$; (c) $K=8$}
\end{figure}

\begin{figure}[H]
\label{fig4}
	\centering
	\begin{subfigure}[t]{0.3\textwidth}
		\includegraphics[width=\textwidth]{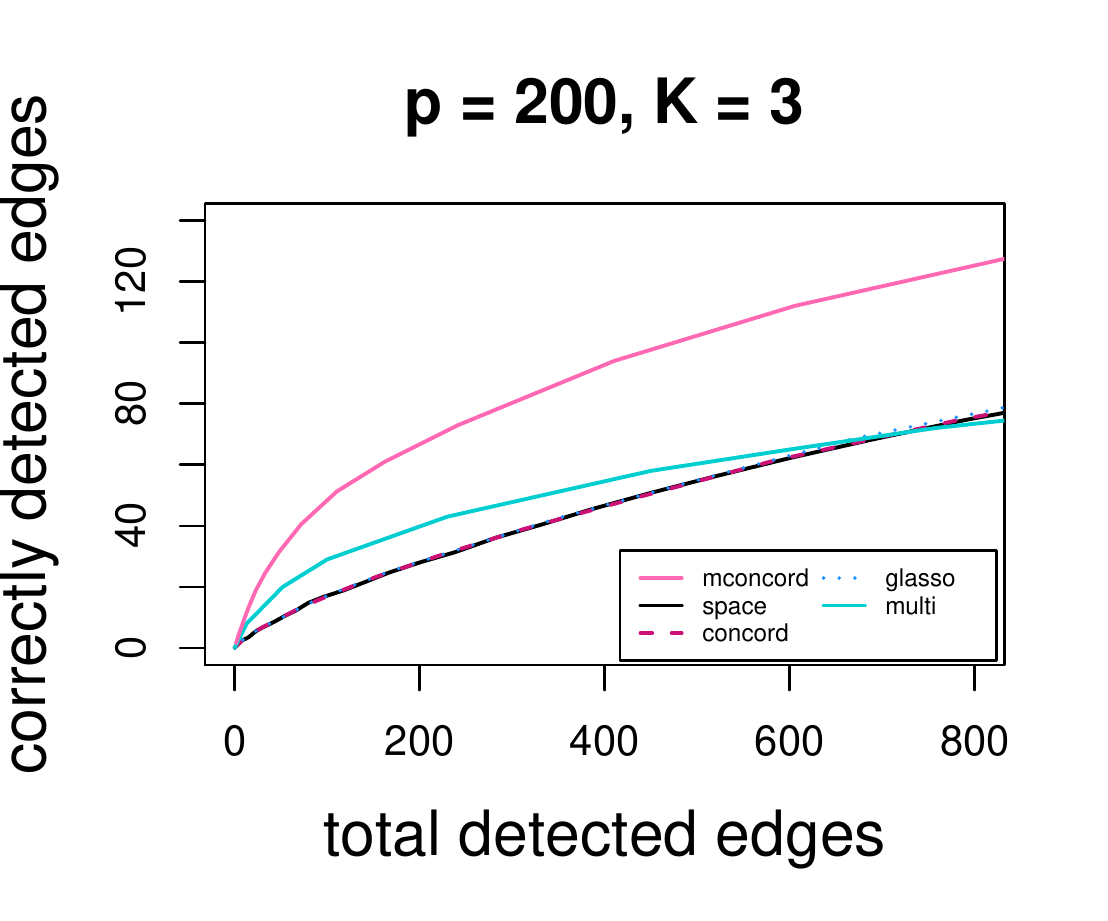}
		\caption{}
	\end{subfigure}
	\begin{subfigure}[t]{0.3\textwidth}
		\includegraphics[width=\textwidth]{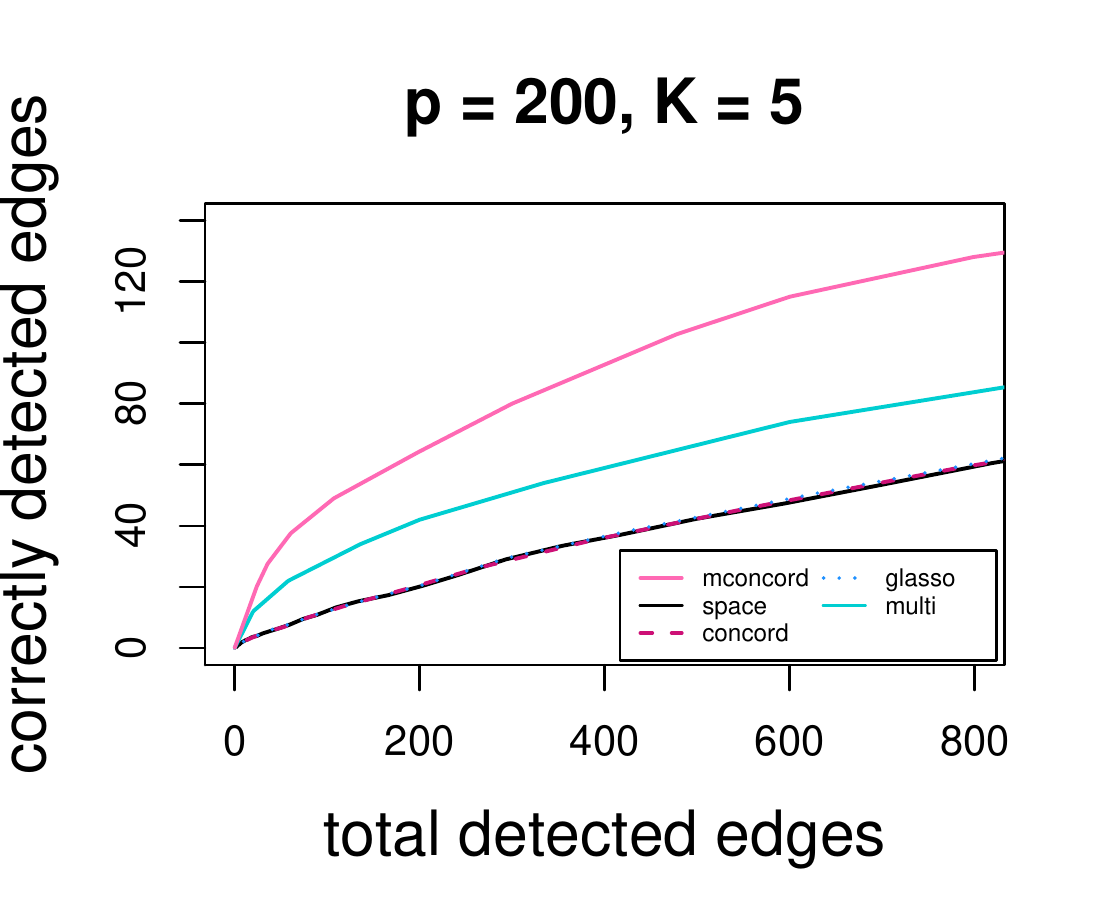}
		\caption{}
	\end{subfigure}
	\begin{subfigure}[t]{0.3\textwidth}
		\includegraphics[width=\textwidth]{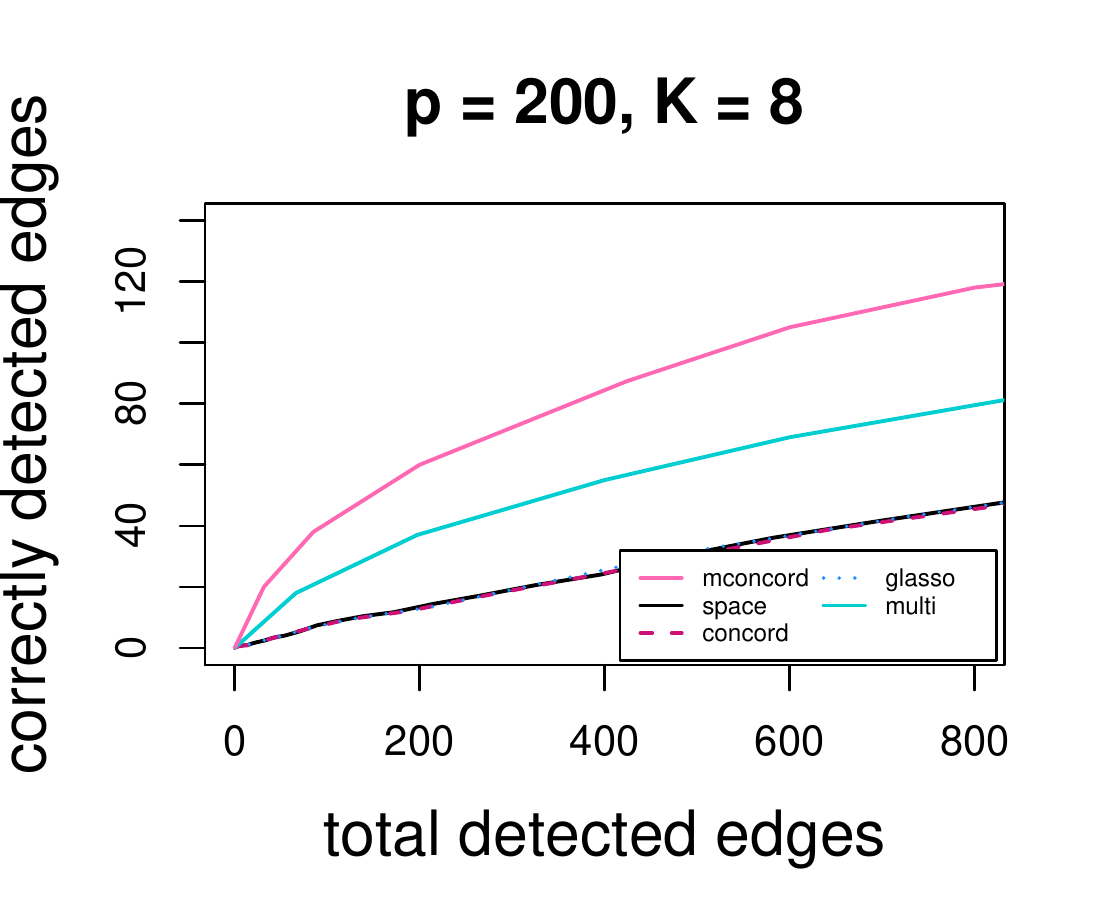}
		\caption{}
	\end{subfigure}
	\caption{Estimation accuracy comparison: total detected edges vs. correctly detected edges with 412 true edges (2\%): (a) $K=3$; (b) $K=5$; (c) $K=8$}
\end{figure}

\begin{figure}[H]
	\label{fig5}
	\centering
	\begin{subfigure}[t]{0.3\textwidth}
		\includegraphics[width=\textwidth]{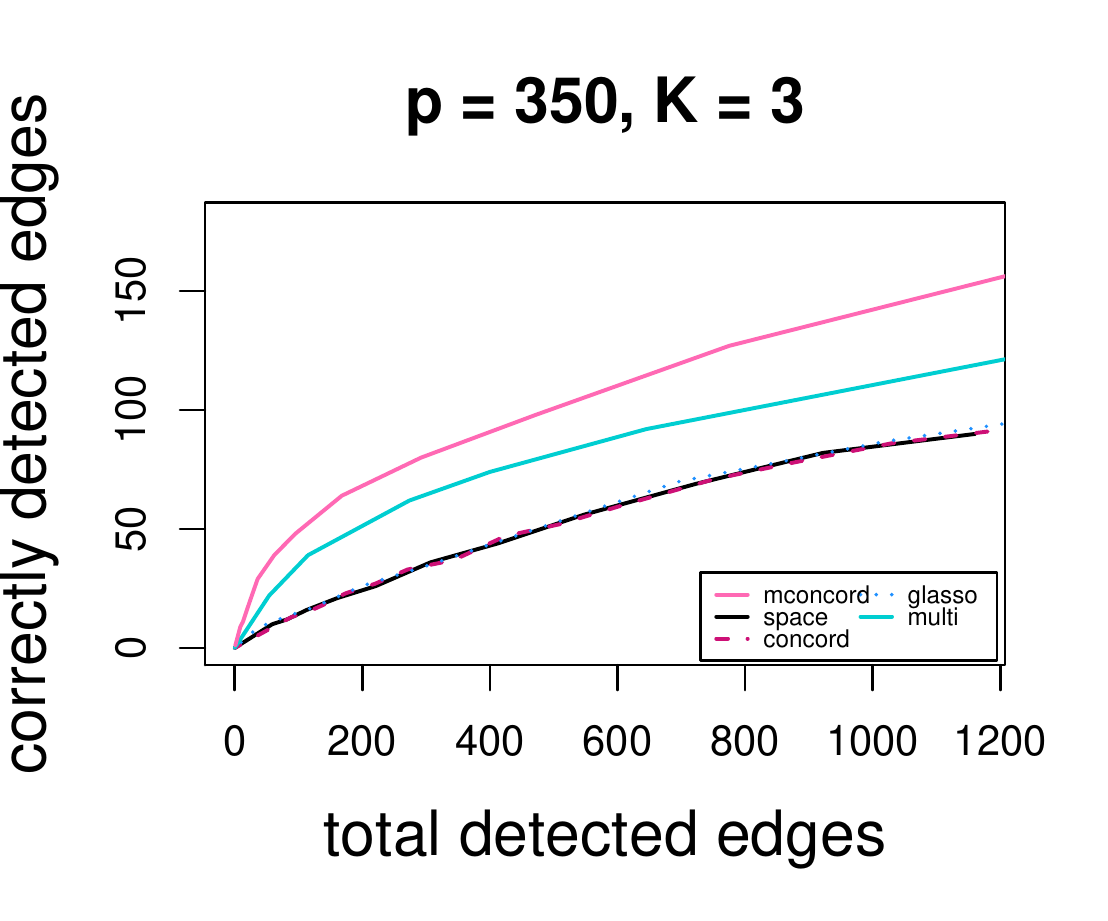}
		\caption{}
	\end{subfigure}
	\begin{subfigure}[t]{0.3\textwidth}
		\includegraphics[width=\textwidth]{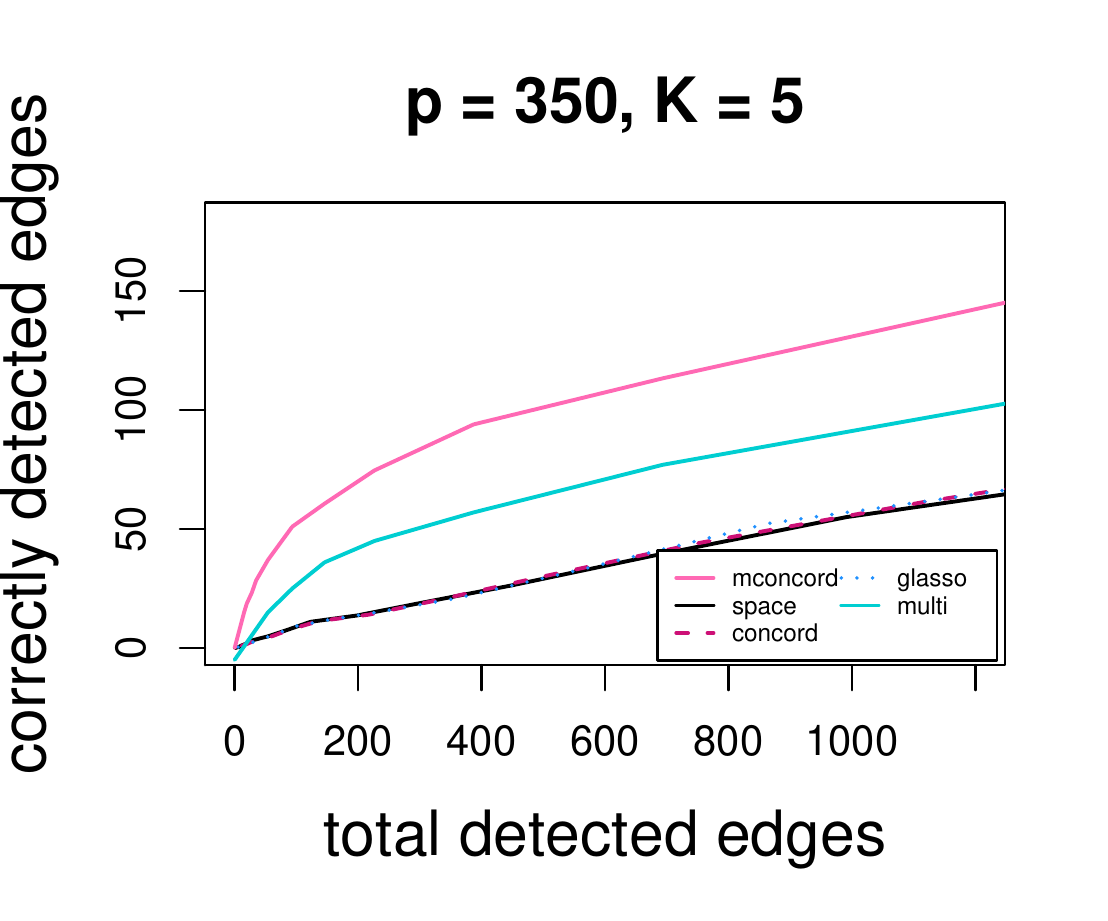}
		\caption{}
	\end{subfigure}
	\begin{subfigure}[t]{0.3\textwidth}
		\includegraphics[width=\textwidth]{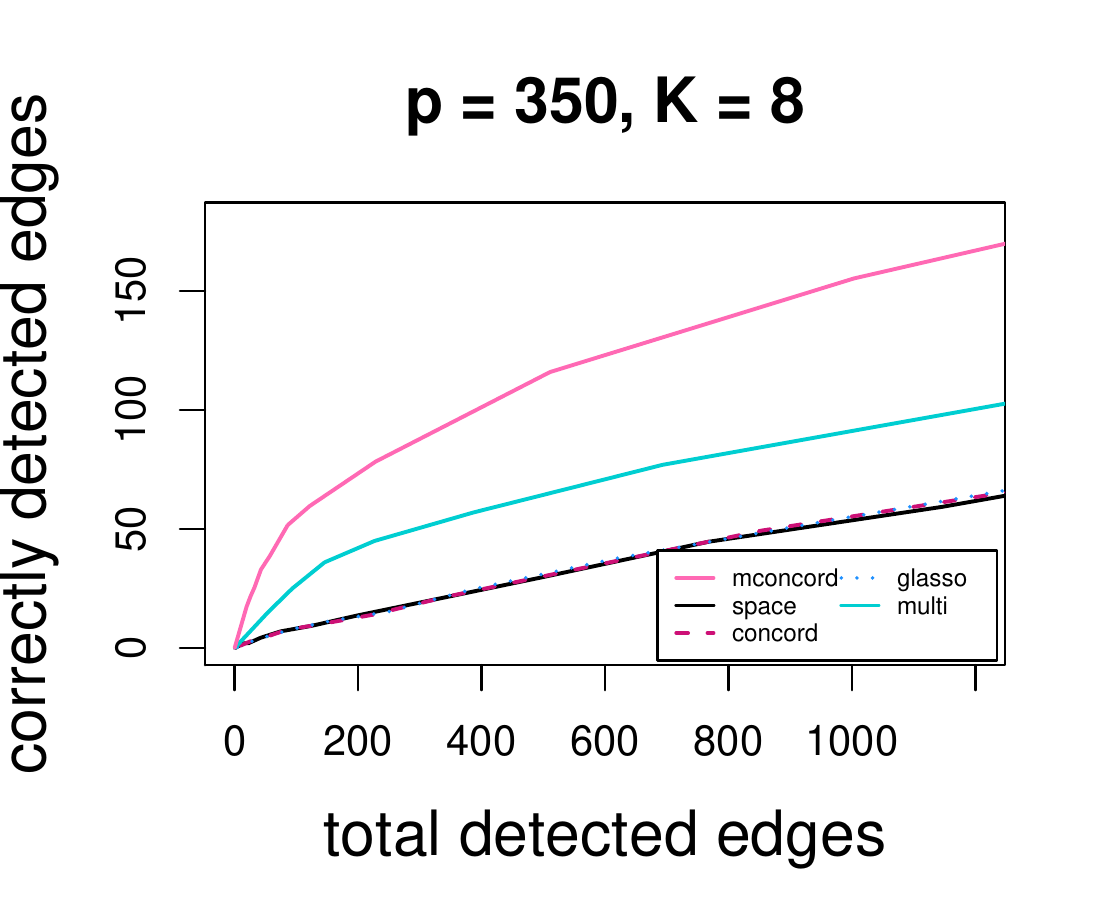}
		\caption{}
	\end{subfigure}
	\caption{Estimation accuracy comparison: total detected edges vs. correctly detected edges with 1250 true edges (2\%): (a) $K=3$; (b) $K=5$; (c) $K=8$}
\end{figure}

We can observe that for all methods, $N_t$ decreases when we increase $\lambda$. It can be seen that \texttt{mconcord} consistently outperforms its counterparts, as it detects more correct edges than the other methods for the same number of total edges detected, especially when we have large $K$ or large $p$. In all scenarios, \texttt{space}, \texttt{concord} and \texttt{glasso} give very similar results. With large $K$ and $p$, \texttt{multi} performs better than univariate methods. 

The better performance of \texttt{moncord} over \texttt{space}, \texttt{concord} and \texttt{glasso} is largely due to the fact that \texttt{mconcord} is designed for multivariate network, and treating the precision matrix by different blocks is more likely to catch an edge even when the signal is comparably weak. On the contrary, the univariate approaches tend to select more unwanted edges since there is high probability that there is at least on non-zero element in the block due to randomness.

In high dimensional settings, regression based methods have simpler quadratic loss function and are computationally faster and more efficient than that of penalized likelihood methods, which optimize with respect to the entire precision matrix at once. The running time for \texttt{mconcord} is about one-third of that for \texttt{multi}. The higher numerical accuracy of regression based methods over penalized likelihood methods were often observed in the univariate setting, and hence is expected to continue in the multivariate setting as well.

\subsection{Model Selection Comparison}
Next in the second study, we compare the model selection performance of the above approaches. We fix $K=4$, and conduct simulation studies for several combinations of $n$ and $p$ with different densities which vary from 41\% to 1\%. The precision matrices are generated using the same technique as in the first study. The tuning parameter $\lambda$ is selected using a 5-fold cross-validation for all methods. We also studied the performance of the Bayesian Information Criterion (BIC) for model selection, but it seems that BIC does not work in the multi dimensional settings. In fact, BIC in most cases tends to choose the smallest model where no edge can be detected.  Here we compare sensitivity (TPR), precision (PPV) and Matthew's Correlation Coefficient (MCC) defined by 
$$
\text{TPR} = \frac{\text{TP}}{\text{TP} + \text{FN}}, 
\text{PPV} = \frac{\text{TP}}{\text{TP} + \text{FP}}, 
\text{MCC} = \frac{\text{TP}\times\text{TN}-\text{FP}\times\text{FN}}{\sqrt{(\text{TP}+\text{FP})(\text{TP}+\text{FN})(\text{TN}+\text{FP})(\text{TN}+\text{FN})}},
$$
where TP, TN, FP and FN denote true positives (number of edges correctly detected), true negatives (number of edges correctly excluded), false positives (number of edges detected but absent in the true model) and false negatives (number of edges falsely excluded). For each network, all final numbers are averaged across 30 independent datasets.

\begin{table}[H]
\label{table1}
	\centering
	\caption{Model selection comparison with $p$ the number of nodes, $q$ the number of true edges and $n$ the sample size with the tuning parameter $\lambda$ optimized by cross-validation. Cases considered below are (i) $p=30$, $q=177$ (41\% density) (ii) $p=50$, $q=137$ (11\% density), (iii) $p=100$, $q=419$ (8\% density), (iv) $p=200$, $q=617$ (3\% density), (v) $p=400$, $q=782$ (1\% density) where the density is $100q/\binom{p}{2}$ in percentage.}
	\begin{tabular}{cccccccc}\hline\hline
		  & $n$ & & \texttt{mconcord} & \texttt{space} & \texttt{concord} & \texttt{glasso} & \texttt{multi}\\\hline
		  \multirow{3}{*}{(i)} & \multirow{3}{*}{50} & $N_t$($N_c$) & 58(34) & 70(35)& 85(42) & 378(157) & 217(89)\\
		  & & TPR(PPV) & 0.19(0.57) & 0.20(0.50) & 0.24(0.49) & 0.89(0.42) & 0.50(0.41)\\
		  & & MCC & 0.14 & 0.08 & 0.09 & 0.04 & 0.01\\\hline
		\multirow{6}{*}{(ii)} & \multirow{3}{*}{50} & $N_t$($N_c$) & 105(57) & 47(10)& 46(9) & 805(105) & 612(69)\\
		& & TPR(PPV) & 0.42(0.54) & 0.07(0.21) & 0.07(0.20) & 0.77(0.13) & 0.50(0.11)\\
		& & MCC & 0.42 & 0.06 & 0.05 & 0.08 & 0.01\\
		& \multirow{3}{*}{100} & $N_t$($N_c$) & 191(64) & 286(58)& 280(59) & 923(122) & 525(69)\\
		 & & TPR(PPV) & 0.47(0.34) & 0.42(0.20) & 0.43(0.21) & 0.89(0.13) & 0.50(0.13)\\
		 & & MCC & 0.30 & 0.16 & 0.17 & 0.11 & 0.05\\\hline
		\multirow{6}{*}{(iii)} & \multirow{3}{*}{100} & $N_t$($N_c$) & 248(87) & 202(40)& 267(51)& 2389(274)& 2501(211)\\
		 & & TPR(PPV) & 0.21(0.35) & 0.10(0.20) & 0.12(0.19) & 0.65(0.11) & 0.50(0.08)\\
		 & & MCC & 0.22 & 0.08 & 0.09 & 0.10 & 0.00\\ 
		 & \multirow{3}{*}{200} & $N_t$($N_c$) & 613(200)& 814(170)& 1005(196)& 1066(204)& 2380(201)\\
		 & & TPR(PPV) & 0.48(0.33) & 0.41(0.21) & 0.47(0.20) & 0.49(0.19) & 0.48(0.08)\\
		 & & MCC & 0.33 & 0.20 & 0.20 & 0.20 & 0.00\\\hline
		\multirow{6}{*}{(iv)} & \multirow{3}{*}{100} & $N_t$($N_c$) & 481(112)& 84(12)& 133(18) & 5657(306)& 4797(240)\\
		 & & TPR(PPV) & 0.18(0.23) & 0.02(0.14) & 0.03(0.14) & 0.50(0.05) & 0.39(0.05)\\
		 & & MCC & 0.18 & 0.04 & 0.05 & 0.08 & 0.06\\
		 & \multirow{3}{*}{200} & $N_t$($N_c$) & 1250(300)& 892(143)& 976(151)& 6357(426)& 4392(226)\\
		 & & TPR(PPV) & 0.49(0.24) & 0.23(0.16) & 0.24(0.15) & 0.69(0.07) & 0.37(0.05) \\
		 & & MCC & 0.31 & 0.16 & 0.16 & 0.14 & 0.06\\\hline
		\multirow{6}{*}{(v)} & \multirow{3}{*}{100} & $N_t$($N_c$) & 764(129)& 31(3)& 54(6)& 14283(326)& 10229(259)\\
		 & & TPR(PPV) & 0.16(0.17) & 0.00(0.10) & 0.00(0.11) & 0.42(0.02) & 0.33(0.03)\\
		 & & MCC & 0.16 & 0.02 & 0.03 & 0.06 & 0.06\\
		 & \multirow{3}{*}{200} & $N_t$($N_c$) & 2063(378)& 396(62)& 404(53)& 16092(480)& 9648(240)\\
		 & & TPR(PPV) & 0.48(0.18) & 0.08(0.16) & 0.07(0.13) & 0.61(0.03) & 0.31(0.02)\\
		 & & MCC & 0.29 & 0.11 & 0.09 & 0.10 & 0.06\\\hline
	\end{tabular}
\end{table}

Table~1 shows that substantial gain is achieved by considering the multivariate aspect in \texttt{mconcord} compared with the univariate methods \texttt{space} and \texttt{concord} in regards of both sensitivity and precision, except for the case $p = 30$ and $n=50$ where these two methods score slightly better TPR due to more selection of edges. Both \texttt{glasso} and \texttt{multi} select very dense models in nearly all cases, and as a consequence their TPR are higher. However, in terms of MCC which accounts for both correct and incorrect selections, \texttt{mconcord} performs consistently better than all the other methods.

\section{Application}
\label{sec:application}
\subsection{Gene/Protein Network Analysis}
According to the NCI website https://dtp.cancer.gov/discovery\_development/nci-60, ``the US National Cancer Institute (NCI) 60 human tumor cell lines screening has greatly served the global cancer research community for more than 20 years. The screening method was developed in the late 1980s as an in vitro drug-discovery tool intended to supplant the use of transplantable animal tumors in anticancer drug screening. It utilizes 60 different human tumor cell lines to identify and characterize novel compounds with growth inhibition or killing of tumor cell lines, representing leukemia, melanoma and cancers of the lung, colon, brain, ovary, breast, prostate, and kidney cancers". 

We apply our method to a dataset from the well-known NCI-60 database, which consists of protein profiles (normalized reverse-phase lysate arrays for 94 antibodies) and gene profiles (normalized RNA microarray intensities from Human Genome U95 Affymetrix chip-set for more than 17000 genes). Our analysis will be restricted to a subset of 94 genes/proteins for which both types of profiles are available. These profiles are available across the same set of 60 cancer cell lines. Each gene-protein combination is represented by its Entrez ID, which is a unique identifier common for a protein and a corresponding gene that encodes this protein. 

Three networks are studied: a network based on protein measurements alone, a network based on gene measurements alone, and a gene-protein multivariate network. For protein alone and gene alone networks, we use \texttt{concord}, and for gene-protein network, we use \texttt{mconcord}. The tuning parameter $\lambda$ is selected using 5-fold cross-validation for all three networks. 

From the gene-protein network 531 edges are selected. For the protein network, 798 edges are selected and for the gene network, 784 edges are selected. Protein and gene-protein networks share 313 edges, while gene and gene-protein networks share 287 edges. However, protein and gene networks only share 167 edges. Table 2 provides summary statistics for these networks.

\begin{table}[H]
	\centering
	\caption{Summary statistics for protein, gene and gene-protein networks}
	\begin{tabular}{lrrr}\hline\hline
		& Protein network & Gene network & Gene-protein network\\\hline
		Number of edges & 798 & 784 & 531\\
		Density (\%) & 18 & 18 & 12\\
		Maximum degree & 24 & 24 & 20\\
		Average node degree & 16.98 & 16.68 & 11.30\\\hline
	\end{tabular}
\end{table}

In Table~3, we also list the top 20 most connected components for all three networks. Among them, the gene-protein network and the protein network share 11, the gene-protein network and the gene network share 10, while the protein network and the gene network share only 6. 

\begin{table}[H]
	\centering
	\caption{Top 20 most connected nodes for three networks (sorted by decreasing degrees)}
	\begin{tabular}{cl|cl|cl}\hline\hline
		\multicolumn{2}{c|}{Gene-protein network} & \multicolumn{2}{|c|}{Protein network} & \multicolumn{2}{|c}{Gene network} \\
		Entrez ID & Gene name & Entrez ID & Gene name & Entrez ID & Gene name\\\hline
		302 & ANXA2 & 4179 & CD46 & 2064 & ERBB2\\
		7280 & TUBB2A & 983 & CDK1 & 5605 & MAP2K2\\
		1398 & CRK & 3265 & HRAS & 307 & ANXA4\\
		4255 & MGMK & 3716 & JAK1 & 5578 & PRKCA\\
		5578 & PRKCA & 10270 & AKAP8 & 1173 & AP2M1\\
		5925 & RB1 & 354 & KLK3 & 1828 & DSG1\\
		9564 & BCAR1 & 1019 & CDK4 & 4179 & CD46\\
		307 & ANXA4 & 6776 & STAT5A & 9961 & MVP\\
		354 & KLK3 & 9564 & BCAR1 & 1000 & CDH2\\
		2064 & ERBB2 & 1398 & CRK & 2932 & GSK3B\\
		4163 & MCC & 3667 & IRS1 & 4176 & MCM7\\
		6778 & STAT6 & 4830 & NME1 & 4436 & MSH2\\
		7299 & TYR & 307 & ANXA4 & 5970 & RELA\\
		1173 & AP2M1 & 1173 & AP2M1 & 999 & CDH1\\
		983 & CDK1 & 2017 & CTTN & 1001 & CDH3\\
		1001 & CDH3 & 4255 & MGMT & 1398 & CRK\\
		1499 & CTNNB1 & 1001 & CDH3 & 2335 & FN1\\
		3716 & JAK1 & 1020 & CDK5 & 5925 & RB1\\
		4179 & CD46 & 3308 & HSPA4 & 7280 & TUBB2A\\
		4830 & NME1 & 4176 & MCM7 & 7299 & TYR\\\hline
	\end{tabular}
\end{table}

\subsection{GDP Network Analysis}
In this analysis, we apply our method to the regional GDP data obtained from U.S. Department of Commerce website https://www.bea.gov/index.html, which contains GDP data including the following 20 different industries with labels: \begin{enumerate*}[(1)] \item utilities (uti), \item construction (cons) , \item Manufacturing (manu), \item Durable goods manufacturing (durable), \item nondurable goods manufacturing (nondu), \item wholesale trade (wholesale), \item retail trade (retail), \item transportation and warehousing (trans), \item information (info), \item finance and insurance (finance), \item real estate and rental and leasing (real), \item professional, scientific and technical services (prof), \item management of companies and enterprises (manage), \item administrative and waste management services (admin), \item educational services (edu), \item health care and social assistance (health), \item arts, entertainment and recreation (arts), \item accommodation and food services (food), \item other services except government (other) and \item government (gov). \end{enumerate*}

The data is available from the first quarter of 2005 to the second quarter of 2016. Data from the third quarter of 2008 to the forth quarter of 2009 is eliminated to reduce the impact of the financial crisis of that period. The data is in 8 regions in the US, including New England (Connecticut, Maine, Massachusetts, New Hampshire, Rhode Island and Vermont), Mideast (Delaware, D.C., Maryland, New Jersey, New York and Pennsylvania), Great Lakes (Illinois, Indiana, Michigan, Ohio and Wisconsin), Plains (Iowa, Kansas, Minnesota, Missouri, Nebraska, North Dakota and South Dakota), Souteast (Alabama, Arkansas, Florida, Georgia, Kentucky, Louisiana, Mississippi, North Carolina, South Carolina, Tennessee, Virginia and West Virginia), Southwest (Arizona, New Mexico, Oklahoma and Texas), Rocky Mountain (Colorado, Idaho, Montana, Utah and Wyoming) and Far West (Alaska, California, Hawaii, Nevada, Oregon and Washington). 

We reduce correlation in the time series data by taking differences of the consecutive observations. A multivariate network consisting of 20 nodes and 8 attributes for each node is studied. After using 5-fold cross-validation to select the tuning parameter $\lambda$, 47 edges are detected, with density of 24.7\% and average node degree of 4.7. The 5 most connected industries are retail trade, transportation, wholesale trade, accommodation and food services, and professional and technical services. The network is shown in Figure~4(a). It is obvious to see hubs comprising of wholesale trade and retail trade. This is very natural for the consumer-driven economy of the US. Both of these two nodes are connected to transportation, as both of these industries heavily rely on transporting goods. Another noticeable fact is that education is connected with government. As part of the services provided by government, it is natural that the quality as well as GDP of educational services can both be influenced by government.

The univariate network using the nationwide GDP data only is also studied for comparison using \texttt{concord}. For the tuning parameter $\lambda$, 5-fold cross-validation is applied, and 95 networks are selected, with density of 50\% and average node degree of 9.5. The 5 most connected industries are administrative and waste management services, accommodation and food services, wholesale trade, professional and technical services and health care and social assistance. The network is shown in Figure~4(b). The more modest degree of connections in the multivariate network seems to be more interpretable.

\begin{figure}[H]
	\label{fig5}
	\centering
	\begin{subfigure}[t]{0.45\textwidth}
		\includegraphics[width=\textwidth]{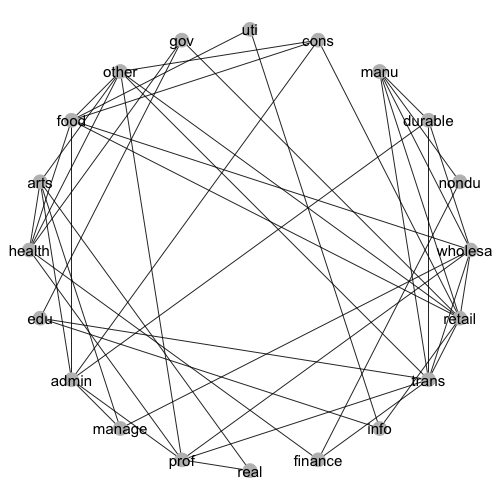}
		\caption{}
	\end{subfigure}
	\begin{subfigure}[t]{0.45\textwidth}
		\includegraphics[width=\textwidth]{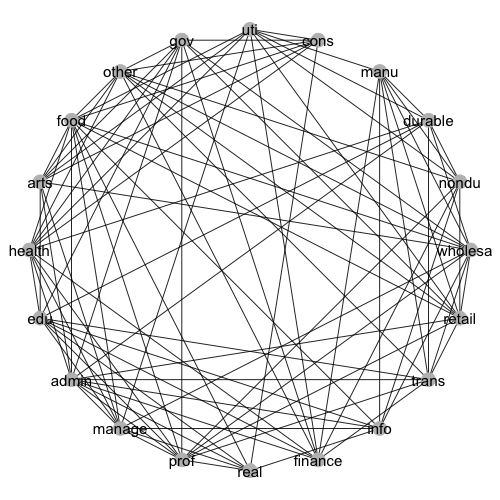}
		\caption{}
	\end{subfigure}
	\caption{Comparison of multivariate and univariate GDP networks}
\end{figure}

\section{Proof of the theorems}
\label{sec:proofs}

We rewrite \eqref{asy1} as $L(\omega,{\sigma},\bm Y)=\frac12\sum_{i=1}^p\sum_{k=1}^{K_i}w_{ik}\Big(Y_{ik}+\sum_{j\neq i}\sum_{l=1}^{K_j}\omega_{ijkl}\tilde{Y}_{jl}\Big)^2$, where $\tilde{Y}_{ik}=Y_{ik}/\sigma^{ik}$. 

For any subset $S\subset T$, the Karush-Kuhn-Tucker (KKT) condition characterizes a solution of the optimization problem
$$
\arg\min\limits_{\omega:\omega_{S^c}=0} \left\{L_n(\omega,\hat{\sigma},Y)+\lambda_n\sum\limits_{1\leq i<j\leq p}\sqrt{\sum\limits_{k=1}^{K_i}\sum\limits_{l=1}^{K_j}\omega_{ijkl}^2}\right\}.
$$
A vector $\hat{\omega}$ is a solution if and only if for any $(i,j,k,l)\in S$
\begin{align*}
&
L'_{n,ijkl}(\hat{\omega},\hat{\sigma},Y) = -\lambda_n\frac{\hat{\omega}_{ijkl}}{\sqrt{\sum\limits_{k',l'}\hat{\omega}_{ijk'l'}^2}},
& \text{if }\exists1\leq k\leq K_i, 1\leq l\leq K_j, \hat{\omega}_{ijkl}\neq0\\
&
\lvert L'_{n,ijkl}(\hat{\omega},\hat{\sigma},Y)\rvert\leq\lambda_n, 
& \text{if }\hat{\omega}_{ijkl}=0, k = 1,\dots,K_i,l=1,\dots,K_j. 
\end{align*}

The following lemmas will be needed in the proof of Theorems~\ref{restricted problem}--\ref{consistency}. Their proofs are deferred to the Appendix.
 
\begin{lemma} 
\label{lemma1} 
The following properties hold.
\begin{enumerate} 
\item [(i)] For all ${\omega}$ and ${\sigma}$, $L(\omega, \sigma,\bm Y)\geq0$. 
\item [(ii)] If $\sigma^{ik}>0$ for all $1\le k\le K_i$ and $i=1,\ldots,p$, then  $L(\cdot,\sigma,Y)$ is convex in $\omega$ and is strictly convex with probability one. 
\item [(iii)] For every index $(i,j,k,l)$ with $i\ne j$,  
$\bar{L}'_{ijkl}(\bar{\omega},\bar{\sigma})=0$.
\item [(iv)] All entries of $\bar{\Sigma}$ are bounded and bounded below. Also, there exist constants $0<\bar{\sigma}_0\leq\bar{\sigma}_{\infty}<\infty$ such that $$\bar{\sigma}_0\leq\min\{\bar{\sigma}^{ik}:1\leq i\leq p,1\leq k\leq K_i\}\leq\max\{\bar{\sigma}^{ik}:1\leq i\leq p,1\leq k\leq K_i\}\leq\bar{\sigma}_{\infty}.$$
\item [(v)] There exists constants  $0<\Lambda^L_{\min}(\bar{\omega},\bar{\sigma})\leq\Lambda^L_{\max}(\bar{\omega},\bar{\sigma})<\infty$, such that
$$
0<\Lambda^L_{\min}(\bar{\omega},\bar{\sigma})\leq\lambda_{\min}(\bar{L}''(\bar{\omega},\bar{\sigma}))\leq\lambda_{\max}(\bar{L}''(\bar{\omega},\bar{\sigma}))\leq\Lambda^L_{\max}(\bar{\omega},\bar{\sigma})<\infty.
$$
\end{enumerate}
\end{lemma}

 \begin{lemma} 
 \label{lemma2}
 \begin{enumerate} 
 	
\item [(i)] There exists a constant $N<\infty$, such that for all $1\leq i\neq j\leq p$ and $1\leq k\leq K_i$, $1\leq l\leq K_j$,  $\bar{L}''_{ijkl,ijkl}(\bar{\omega},\bar{\sigma})\leq N$. 

\item [(ii)] There exists constants $M_1,M_2<\infty$, such that for any $1\leq i< j\leq p$,
$$
\mathrm{Var}(L'_{ijkl}(\bar{\omega},\bar{\sigma},Y))\leq M_1,\quad \mathrm{Var}(L''_{ijkl,ijkl}(\bar{\omega},\bar{\sigma},Y))\leq M_2.
$$
\item [(iii)] There exists a positive constant $g$, such that for all $(i,j,k,l)\in \mathcal{A}$,
$$
L''_{ijkl,ijkl}(\bar{\omega},\bar{\sigma})-L''_{ijkl,\mathcal{A}_{-ijkl}}(\bar{\omega},\bar{\sigma})\Big[L''_{\mathcal{A}_{-ijkl},\mathcal{A}_{-ijkl}}(\bar{\omega},\bar{\sigma})\Big]^{-1}L''_{\mathcal{A}_{ijkl},ijkl}(\bar{\omega},\bar{\sigma})\geq g,
$$
where $\mathcal{A}_{-ijkl}=\mathcal{A}\setminus\{(i,j,k,l)\}$. 
\item [(iv)] For any $(i,j,k,l)\in \mathcal{A}^c$,
$
\lVert\bar{L}''_{ijkl,\mathcal{A}}(\bar{\omega},\bar{\sigma})[\bar{L}''_{\mathcal{A},\mathcal{A}}(\bar{\omega},\bar{\sigma})]^{-1}\rVert_2\leq M_3.
$ for some constant $M_3$. 
\end{enumerate}
\end{lemma}

\begin{lemma}
\label{lemma3}
There exists a constant $M_4<\infty$, such that for any $1\leq i\leq j\leq p$ and $1\leq k\leq K_i$, $1\leq l\leq K_j$, $\lVert\mathrm{E}[Y_{ik}Y_{jl}\tilde{Y}\tilde{Y}^T]\rVert\leq M_4$.
\end{lemma}

\begin{lemma} 
\label{lemma4}
Let the conditions of Theorem~\ref{selection} hold. Then for any sequence $C_n\rightarrow\infty$,
$$
\max\limits_{1\leq i< j\leq p,1\leq k,l\leq K}\Big|L'_{n,ijkl}(\bar{\omega},\bar{\sigma},\bm{Y})-L'_{n,ijkl}(\bar{\omega},\hat{\sigma},\bm{Y})\Big|\leq C_n\sqrt{\frac{\log n}{n}},
$$
$$
\max_{i<j, t<s}\Big|L''_{n,ijkl,tsk'l'}(\bar{\omega},\bar{\sigma},\bm{Y})-L''_{n,ijkl,tsk'l'}(\bar{\omega},\hat{\sigma},\bm{Y})\Big|\leq C_n\sqrt{\frac{\log n}{n}},
$$
hold with probability tending to $1$.
\end{lemma}

\begin{lemma} 
\label{lemma5}	
If $K_{\max}^2q_n= o(\sqrt{{n}/{\log n}})$, then for any sequence $C_n\rightarrow\infty$ and any $u\in\mathbb{R}^{\lvert\mathcal{A}\rvert}$, the following hold with probability tending to $1$:
\begin{eqnarray*} 
\lVert L'_{n,\mathcal{A}}(\bar{\omega},\hat{\sigma},\bm{Y})\rVert_2 &\leq & C_nK_{\max}\sqrt{\frac{q_n\log n}{n}},\\
\lvert u^TL'_{n,\mathcal{A}}(\bar{\omega},\hat{\sigma},\bm{Y})\rvert &\leq & C_n\lVert u\rVert_2K_{\max}\sqrt{\frac{q_n\log n}{n}},\\
\lvert u^TL''_{n,\mathcal{A},\mathcal{A}}(\bar{\omega},\hat{\sigma},\bm{Y})u-u^T\bar{L}''_{\mathcal{A},\mathcal{A}}(\bar{\omega},\bar{\sigma})u\rvert &\leq & C_n\lVert u\rVert_2^2K_{\max}^2q_n\sqrt{\frac{\log n}{n}},\\
\lVert L''_{n,\mathcal{A},\mathcal{A}}(\bar{\omega},\hat{\sigma},\bm{Y})u-\bar{L}''_{\mathcal{A},\mathcal{A}}(\bar{\omega},\bar{\sigma})u\rVert_2 &\leq & C_n\lVert u\rVert_2K_{\max}^2q_n\sqrt{\frac{\log n}{n}}.
\end{eqnarray*}
\end{lemma}

\begin{lemma} 
\label{lemma6}	
	Assume that the conditions of Theorem~\ref{restricted problem} hold. Then exists a constant $\bar{C}_1>0$, such that with probability tending to $1$, there exists a local minimum  of the restricted problem \eqref{thm1} within the disc
$\{\omega:\lVert\omega-\bar{\omega}\rVert_2\leq \bar{C}_1K_{\max}\sqrt{q_n}\lambda_n\}$. 
\end{lemma}

\begin{lemma} 
\label{lemma7} 
Assume the conditions of Theorem~\ref{restricted problem}. Then exists a constant $\bar{C}_2>0$ such that for any $\omega$ satisfying $\lVert\omega-\bar{\omega}\rVert_2\geq \bar{C}_2K_{\max}\sqrt{q_n}\lambda_n$ and $\omega_{\mathcal{A}^c}=0$, we have $\lVert L'_{n,\mathcal{A}}(\omega,\hat{\sigma},\bm{Y})\rVert_2>K_{\max}\sqrt{q_n}\lambda_n$ with probability tending to $1$.
\end{lemma}

\begin{lemma}
	\label{lemma8}
Let $D_{\mathcal{A},\mathcal{A}}(\bar{\omega},\bar{\sigma},Y)=L''_{\mathcal{A},\mathcal{A}}(\omega, \bar{\sigma},Y)-\bar{L}''_{\mathcal{A},\mathcal{A}}(\bar{\omega},\bar{\sigma})$. Then there exists a constant $M_5<\infty$, such that for any $(i,j,k,l)\in \mathcal{A}$, $\lambda_{\max}(\mathrm{Var}(D_{\mathcal{A},ijkl}(\bar{\omega},\bar{\sigma},Y)))\leq M_5$. 
\end{lemma}

\begin{proof}[of Theorem~\ref{restricted problem}] 
\rm	
The existence of a solution of \eqref{thm1} follows from Lemma~\ref{lemma6}. By the KKT condition, any solution $\hat{\omega}$ of \eqref{thm1}, satisfies $\Vert L'_{n,\mathcal{A}}(\hat{\omega},\hat{\sigma},\bm{Y})\rVert_{\infty}\leq\lambda_n$,  implying\\
$\Vert L'_{n,\mathcal{A}}(\hat{\omega},\hat{\sigma},\bm{Y})\rVert_2\leq K_{\max}\sqrt{q_n}\Vert L'_{n,\mathcal{A}}(\hat{\omega},\hat{\sigma},\bm{Y})\rVert_{\infty}\leq K_{\max}\sqrt{q_n}\lambda_n$. Thus by Lemma~\ref{lemma7}, with probability tending to 1, all solutions of \eqref{thm1} are inside the disc $\{\omega: \lVert\omega-\bar{\omega}\rVert_2\leq \bar{C}_2K_{\max}\sqrt{q_n}\lambda_n\}$. Hence with probability tending to $1$, 
$
\Vert\hat{\omega}_{\mathcal{A}}^{\lambda_n}-\bar{\omega}_\mathcal{A}\rVert_2\leq \bar{C}_2(\bar{\omega})K_{\max}\sqrt{q_n}\lambda_n$.
\end{proof}

\begin{proof}[of Theorem~\ref{selection}]
\rm
By the KKT condition and the expansion of $L'_{n,\mathcal{A}}(\hat{\omega}_{\mathcal{A}}^{\lambda_n},\hat{\sigma},\bm{Y})$ at $\bar{\omega}$, 
\begin{align*}
-\lambda_n\hat{M}^{\mathcal{A}} & = L'_{n,\mathcal{A}}(\hat{\omega}_{\mathcal{A}}^{\lambda_n},\hat{\sigma},\bm{Y}) = L'_{n,\mathcal{A}}(\bar{\omega},\hat{\sigma},\bm{Y})+L''_{n,\mathcal{A},\mathcal{A}}(\bar{\omega},\hat{\sigma},\bm{Y})\nu_n\\
& = \bar{L}''_{\mathcal{A},\mathcal{A}}(\bar{\omega},\bar{\sigma})\nu_n+L'_{n,\mathcal{A}}(\bar{\omega},\hat{\sigma},\bm{Y})+\Big[L''_{n,\mathcal{A},\mathcal{A}}(\bar{\omega},\hat{\sigma},\bm{Y})-\bar{L}''_{\mathcal{A},\mathcal{A}}(\bar{\omega},\bar{\sigma})\Big]\nu_n,
\end{align*}
where $\nu_n:=\hat{\omega}_{\mathcal{A}}^{\lambda_n}-\bar{\omega}_\mathcal{A}$ and $\hat{M}^{\mathcal{A}}=({\hat{\omega}_{ijkl}}/{\sqrt{\sum\limits_{k',l'}\hat{\omega}_{ijk'l'}^2}}: (i,j,k,l)\in \mathcal{A})^T$. Therefore 
\begin{equation} 
\label{nu}
\nu_n = -\lambda_n[\bar{L}''_{\mathcal{A},\mathcal{A}}(\bar{\omega},\bar{\sigma})]^{-1}\hat{M}^{\mathcal{A}}-[\bar{L}''_{\mathcal{A},\mathcal{A}}(\bar{\omega},\bar{\sigma})]^{-1}\Big[L'_{n,\mathcal{A}}(\bar{\omega},\hat{\sigma},\bm{Y})+D_{n,\mathcal{A},\mathcal{A}}(\bar{\omega},\hat{\sigma},\bm{Y})\nu_n\Big],
\end{equation}
where $D_{n,\mathcal{A},\mathcal{A}}(\bar{\omega},\hat{\sigma},\bm{Y})=L''_{n,\mathcal{A},\mathcal{A}}(\bar{\omega},\hat{\sigma},\bm{Y})-\bar{L}''_{\mathcal{A},\mathcal{A}}(\bar{\omega},\bar{\sigma})$. Next, fix $(i,j,k,l)\in \mathcal{A}^c$, and consider the expansion of $L'_{n,ijk}(\hat{\omega}_{\mathcal{A}}^{\lambda_n},\hat{\sigma},\bm{Y})$ around $\bar{\omega}$ is given by 
\begin{eqnarray} 
\label{exp}
\lefteqn{ L'_{n,ijkl}(\bar{\omega},\hat{\sigma},\bm{Y})+L''_{n,ijkl,A}(\bar{\omega},\hat{\sigma},\bm{Y})\nu_n}\nonumber\\
&& = L'_{n,ijkl}(\bar{\omega},\hat{\sigma},\bm{Y}) + \bar{L}''_{ijkl,\mathcal{A}}(\bar{\omega},\bar{\sigma})\nu_n + \Big[L''_{n,ijkl,\mathcal{A}}(\bar{\omega},\hat{\sigma},\bm{Y})-\bar{L}''_{ijkl,\mathcal{A}}(\bar{\omega},\bar{\sigma})\Big]\nu_n\nonumber\\
&& = L'_{n,ijkl}(\bar{\omega},\hat{\sigma},\bm{Y})+\bar{L}''_{ijkl,\mathcal{A}}(\bar{\omega},\bar{\sigma})\nu_n+D_{n,ijkl,\mathcal{A}}(\bar{\omega},\hat{\sigma},\bm{Y})\nu_n.
\end{eqnarray}
Then plugging \eqref{nu} into \eqref{exp} and rearranging, $L'_{n,ijkl}(\hat{\omega}_{\mathcal{A}}^{\lambda_n},\hat{\sigma},\bm{Y})$ is given by 
\begin{eqnarray} 
\label{nu1}
\lefteqn{ L'_{n,ijkl}(\bar{\omega},\hat{\sigma},\bm{Y}) -\lambda_n\bar{L}''_{ijkl,\mathcal{A}}(\bar{\omega},\bar{\sigma})[\bar{L}''_{\mathcal{A},\mathcal{A}}(\bar{\omega},\bar{\sigma})]^{-1}\hat{M}^{\mathcal{A}}}\nonumber\\
&& -\bar{L}''_{ijkl,\mathcal{A}}(\bar{\omega},\bar{\sigma})[\bar{L}''_{\mathcal{A},\mathcal{A}}(\bar{\omega},\bar{\sigma})]^{-1}L'_{n,\mathcal{A}}(\bar{\omega},\hat{\sigma},\bm{Y})\nonumber\\
&& +\Big[D_{n,ijkl,\mathcal{A}}(\bar{\omega},\hat{\sigma},\bm{Y})-\bar{L}''_{ijkl,\mathcal{A}}(\bar{\omega},\bar{\sigma})[\bar{L}''_{\mathcal{A},\mathcal{A}}(\bar{\omega},\bar{\sigma})]^{-1}D_{n,\mathcal{A},\mathcal{A}}(\bar{\omega},\hat{\sigma},\bm{Y})\Big]\nu_n.
\end{eqnarray}
By Condition C2, for any $(i,j,k,l)\in \mathcal{A}^c:|\bar{L}''_{ijkl,\mathcal{A}}(\bar{\omega},\bar{\sigma})[\bar{L}''_{\mathcal{A},\mathcal{A}}(\bar{\omega},\bar{\sigma})]^{-1}M|\leq\delta<1$. By Theorem~\ref{restricted problem}, we have $\lVert\hat{\omega}_{\mathcal{A}}^{\lambda_n}-\bar{\omega}_{\mathcal{A}}\Vert_2=O_p(K_{\max}\sqrt{q_n}\lambda_n)=o_p(1)$, then $\lvert\hat{M}^{\mathcal{A}}-M\rvert=o_p(1)$. Hence for any $(i,j,k,l)\in \mathcal{A}^c:|\bar{L}''_{ijkl,\mathcal{A}}(\bar{\omega},\bar{\sigma})[\bar{L}''_{\mathcal{A},\mathcal{A}}(\bar{\omega},\bar{\sigma})]^{-1}\hat{M}^{\mathcal{A}}|\leq\delta<1$. 
Thus it suffices to prove that the remaining term in \eqref{nu1} are  $o(\lambda_n)$ with probability tending to 1 uniformly for all $(i,j,k,l)\in \mathcal{A}^c$. Then since $|\mathcal{A}^c|\leq K_{\max}^2p^2=O(n^{2\kappa})$, the event
$
\max_{(i,j,k,l)\in \mathcal{A}^c}|L'_{n,ijkl}(\hat{\omega}_{\mathcal{A}}^{\lambda_n},\hat{\sigma},\bm{Y})|<\lambda_n
$ 
happens with probability tending to 1. \\
By Lemma~\ref{lemma2}(iv), for any $(i,j,k,l)\in \mathcal{A}^c$, $\lVert\bar{L}''_{ijkl,\mathcal{A}}(\bar{\omega},\bar{\sigma})[\bar{L}''_{\mathcal{A},\mathcal{A}}(\bar{\omega},\bar{\sigma})]^{-1}\rVert_2\leq M_3(\bar{\omega},\bar{\sigma})$. Therefore by Lemma~\ref{lemma5}, 
$$
\max_{(i,j,k,l)\in \mathcal{A}^c}|\bar{L}''_{ijkl,\mathcal{A}}(\bar{\omega},\bar{\sigma})[\bar{L}''_{\mathcal{A},\mathcal{A}}(\bar{\omega},\bar{\sigma})]^{-1}L'_{n,A}(\bar{\omega},\hat{\sigma},\bm{Y})|\leq C_nK_{\max}\sqrt{\frac{q_n\log n}{n}}=o(\lambda_n)
$$
with probability tending to 1, choosing a sufficiently slow $C_n\rightarrow\infty$. 
By Lemma~\ref{lemma2}(ii), $\text{Var}(L'_{ijkl}(\bar{\omega},\bar{\sigma},{Y}))\leq M_1(\bar{\omega},\bar{\sigma})$. Then as in Lemma~\ref{lemma5}, with probability tending to 1, $\max_{i,j,k,l}|L'_{n,ijkl}(\bar{\omega},\hat{\sigma},\bm{Y})|\leq C_n\sqrt{({\log n})/{n}}=o(\lambda_n)
$, by virtue of the assumption that $\lambda_n\sqrt{n/{\log n}}\rightarrow\infty$. \\
Note that by Theorem~\ref{restricted problem}, $\lVert\nu_n\rVert_2\leq C_nK_{\max}\sqrt{q_n}\lambda_n$ with probability tending to 1. Thus as in Lemma~\ref{lemma5}, for sufficiently slowly growing sequence $C_n\rightarrow\infty$,  $|D_{n,ijkl,A}(\bar{\omega},\hat{\sigma},\bm{Y})\nu_n|\leq C_nK_{\max}\sqrt{{q_n(\log n)}/{n}}K_{\max}\sqrt{q_n}\lambda_n=o(\lambda_n)$ with probability tending to 1. This claim follows from the assumption $K_{\max}^2q_n=o(\sqrt{n/{\log n}})$.\\
Finally, let $b^T=\bar{L}''_{ijkl,\mathcal{A}}(\bar{\omega},\bar{\sigma})[\bar{L}''_{\mathcal{A},\mathcal{A}}(\bar{\omega},\bar{\sigma})]^{-1}$. By the Cauchy-Schwartz inequality
\begin{eqnarray*}
|b^TD_{n,\mathcal{A},\mathcal{A}}(\bar{\omega},\bar{\sigma},\bm{Y})\nu_n| &\leq &\lVert b^TD_{n,\mathcal{A},\mathcal{A}}(\bar{\omega},\bar{\sigma},\bm{Y})\rVert_2\lVert\nu_n\rVert_2\\
&\leq & K_{\max}^2q_n\lambda_n\max_{(i',j',k',l')\in \mathcal{A}}|b^TD_{n,\mathcal{A},i'j'k'l'}(\bar{\omega},\bar{\sigma},\bm{Y})|.
\end{eqnarray*}
In order to show that the right hand side is $o(\lambda_n)$ with probability tending to 1, it suffices to show 
$$\max_{(i',j',k','l)\in \mathcal{A}}|b^TD_{n,\mathcal{A},i'j'k'l'}(\bar{\omega},\bar{\sigma},\bm{Y})|=O\left(\sqrt{\frac{\log n}{n}}\right)$$
with probability tending to 1, because of the assumption $K_{\max}^2q_n=o(\sqrt{{n}/{\log n}})$. This is implied by
$\text{E}(|b^TD_{\mathcal{A},i'j'k'l'}(\bar{\omega},\bar{\sigma},Y)|^2)\leq \lVert b\rVert_2^2\lambda_{\max}(\text{Var}(D_{\mathcal{A},i'j'k'l'}(\bar{\omega},\bar{\sigma},Y)))$
being bounded, which follows immediately from Lemma~\ref{lemma1}(iv) and Lemma~\ref{lemma8}. Finally, as in Lemma~\ref{lemma5},
\begin{eqnarray*}
|b^TD_{n,\mathcal{A},\mathcal{A}}(\bar{\omega},\hat{\sigma},\bm{Y})\nu_n| &\leq & |b^TD_{n,\mathcal{A},\mathcal{A}}(\bar{\omega},\bar{\sigma},\bm{Y})\nu_n|\\  
&&+|b^T(D_{n,\mathcal{A},\mathcal{A}}(\bar{\omega},\bar{\sigma},\bm{Y})-D_{n,\mathcal{A},\mathcal{A}}(\bar{\omega},\hat{\sigma},\bm{Y}))\nu_n|,
\end{eqnarray*}
where by Lemma~\ref{lemma4}, the second term on the right hand side is bounded by $O_p(\sqrt{({\log n})/{n}})\lVert b\rVert_2\lVert\nu_n\rVert_2$. Note that $\Vert b\rVert_2=O(K_{\max}\sqrt{q_n})$, thus the second term is also of order $o(\lambda_n)$ by the assumption $K_{\max}^2q_n= o(\sqrt{{n}/{\log n}})$. 
\end{proof}
\begin{proof}[of Theorem~\ref{consistency}]
\rm
By Theorems~\ref{restricted problem} and \ref{selection} and the KKT condition, with probability tending to 1, a solution of the restricted problem is also a solution of the original problem. This shows the existence of the desired solution. For part (ii), the assumed condition on the signal strength implies that missing a signal costs more than the estimation error in part (i), and hense it will be impossible to miss such a signal. This shows the selection consistency. If the objective function is strictly convex, the solution is also unique, so this will be the only solution for the original problem. 
\end{proof}
Finally, convergence properties of the estimator of $\sigma$ claimed in Proposition~\ref{pro:sigma} is shown. 

\begin{proof}[of Proposition~\ref{pro:sigma}]
\rm
Observe that when $\sum\limits_{i=1}^pK_i<\beta n$, $\bm{e}_{ik}$ can be expressed as
$
\bm{e}_{ik}=\bm{Y}_{ik}-\bm{Y}_{-ik}^T(\bm{Y}_{-ik}^T\bm{Y}_{-ik})^{-1}\bm{Y}_{-ik}\bm{Y}_{ik}$. 
As argued in Peng et al. \cite{space},  $\text{E}(\bm{e}_{ik}^T\bm{e}_{ik})=1/\bar{\sigma}^{ik}$. Therefore, by Lemma~\ref{app1} of the Appendix and Lemma~\ref{lemma1}(iv), we have ${\max}\{\lvert\hat{\sigma}^{ik}-\bar{\sigma}^{ik}\rvert: {1\leq k\leq K_i, 1\leq i\leq p}\}=O_p(\sqrt{({\log n})/{n}})$.
\end{proof}

\section*{Acknowledgement}

This research is partially supported by National Science Foundation grant  DMS-1510238. 

\appendix

\section*{Appendix A. Proof of the lemmas}
\label{sec:appendix lemmas}

\begin{proof}[of Lemma~\ref{lemma1}]
\rm
The assertions (i) and (ii) are self-evident from the definition of $L$. To prove (iii), denote the residual for the $i$th term by
$e_{ik}(\omega,\sigma)=Y_{ik}+\sum\limits_{j\neq i}\sum\limits_{l=1}^{K_j}\omega_{ijkl}\tilde{Y}_{jl}$. 
Then evaluated at the true parameter values $(\bar{\omega},\bar{\sigma})$, we have $e_{ik}(\bar{\omega},\bar{\sigma})$ uncorrelated with $Y_{jl}$ and $\E(e_{ik}(\bar{\omega},\bar{\sigma}))=0$. Since $
{\partial L(\omega,\sigma,Y)}/{\partial\omega_{ijkl}}=w_{ik}e_{ik}(\omega,\sigma)Y_{jl}+w_{jl}e_{jl}(\omega,\sigma)Y_{ik}$, (iii) follows by taking expectation. \\
Since all eigenvalues of $\bar{\Sigma}$ lie between two positive numbers, so do all diagonal entries because these are values of quadratic forms for unit vectors having $1$ at one place. All off-diagonal entries lie in $[-\Lambda_{\max}, \Lambda_{\max}]$ because these are values of bilinear forms at such unit vectors. This shows (iv).\\ 
To prove (v), let $\tilde{X}=(\tilde{X}_{(11,21)},\dots,\tilde{X}_{(11,2K_2)},\dots,\tilde{X}_{(1K_1,2K_2)},\dots,\tilde{X}_{((p-1)K_{p-1},pK_p)})$, with\\
 $\tilde{X}_{(ik,jl)} =(0,\dots,0,\tilde{Y}_{jl},0,\dots,0,\tilde{Y}_{ik},0,\dots,0)^T$, a matrix of order  $p\sum_{i=1}^pK_i \times \sum_{i<j}K_iK_j$, where only the $(i,k)$th and $(j,l)$th elements are non zero. The loss function can be written as $L(\omega, \sigma, Y)=\frac12\lVert w^{1/2}(Y-\tilde{X}\omega)\rVert_2^2$, where $w^{1/2}=\text{diag}(\sqrt{w_{11}},\dots,\sqrt{w_{pK_p}})$. Thus $\bar{L}''(\omega, \sigma)=\E[\tilde{X}^T(w^{1/2})^2\tilde{X}]$. Let $d=\sum_{i<j}K_iK_j$, the number of columns in $\tilde{X}$, and denote its $(i,k)$th row by $X_{ik}^T$, $1\leq k\leq K_i$, $1\leq i\leq p$. Then for any unit vector $a\in \mathbb{R}^{d}$, we have
$$
a^T\bar{L}''(\bar{\omega},\bar{\sigma})a=\E(a^T\tilde{X}^T(w^{1/2})^2\tilde{X}a)=\E\bigg(\sum\limits_{i=1}^p\sum\limits_{k=1}^{K_i}w_{ik}(X_{ik}^Ta)^2\bigg).
$$
Index the elements of $a$ as $(a_{(11,21)},\dots,a_{(11,2K_2)},\dots,a_{(1K_1,2K_2)},\dots,a_{((p-1)K_{p-1},pK_p)})^T$, and for each $1\leq i\leq p$ and $1\leq k\leq K_i$, define $a_{ik}\in \mathbb{R}^{K_ip}$ by 
$$
a_{ik} = \begin{cases} (0,\dots,0,a_{(1k,21)},\dots,a_{(1k,2K_2)},\dots,a_{(1k,p1)},\dots,a_{(1k,pK_p)})^T, & i = 1,\\
(a_{(pk,11)},\dots,a_{(pk,1K_1)},\dots,a_{(pk,(p-1)1)},\dots,a_{(pk,(p-1)K_{p-1})},0,\dots,0)^T, & i=p,\\
(a_{(11,ik)},\dots,a_{((i-1)K_{i-1},ik)},0,\dots,0,a_{(ik,(i+1)1)},\dots,a_{(ik,pK_p)})^T, & 1<i<p,
\end{cases}
$$
with exactly $K_i$ zeros and $\sum_{j\neq i}K_j$ non-zeros. Then by definition $X_{ik}^Ta=\tilde{Y}^Ta_{ik}$. Also note that $\sum\limits_{i=1}^p\sum\limits_{k=1}^{K_i}\lVert a_{ik}\rVert_2^2=2\lVert a\rVert_2^2=2$. This is because, each element of $a_{ik}$ appears exactly twice in $a$. Therefore, since $\bar{L}''(\bar{\omega},\bar{\sigma})=\E\tilde{Y}\tilde{Y}^T $, we have 
$$
a^T\bar{L}''(\bar{\omega},\bar{\sigma})a=\sum\limits_{i=1}^p\sum\limits_{k=1}^{K_i}w_{ik}a_{ik}^T\tilde{\Sigma} a_{ik}\geq\sum\limits_{i=1}^p\sum\limits_{k=1}^{K_i}w_{ik}\lambda_{\min}(\tilde{\Sigma})\lVert a_{ik}\rVert_2^2\geq2w_0\lambda_{\min}(\tilde{\Sigma}),
$$
 where $\tilde{\Sigma}=\text{var}(\tilde{Y})$. Similarly, $a^T\bar{L}''(\bar{\omega})a\leq2w_{\infty}\lambda_{\max}(\tilde{\Sigma})$. By Condition C1, $\tilde{\Sigma}$ has bounded eigenvalues, and hence (v) follows. 
 \end{proof}
 
 \begin{proof}[of Lemma~\ref{lemma2}]
 \rm
 The proof of (i) follows because $\bar{L}''_{ijkl,i'j'k'l'}(\bar{\omega},\bar{\sigma})=\sigma_{jl,j'l'}+\sigma_{ik,i'k'}$, and the entries of $\bar \Sigma$ are bounded by Lemma~\ref{lemma1}(iv). \\
 For (ii) note that $\text{Var}(e_{ik}(\bar\omega,\bar\sigma))=1/\bar{\sigma}^{ik}$ and $\text{Var}(Y_{ik})=\bar{\sigma}_{ik,ik}$,
 \begin{align*}
 \text{Var}(L'_{n,ijkl}(\bar{\omega},\bar{\sigma},Y)) & = \text{Var}(w_{ik}e_{ik}(\bar\omega,\bar\sigma)Y_{jl})+\text{Var}(w_{jl}e_{jl}(\bar\omega,\bar\sigma)Y_{ik})\\
 & \leq \E(w_{ik}^2e^2_{ik}(\bar\omega,\bar\sigma)Y_{jl}^2)+\E(w_{jl}^2e^2_{jl}(\bar\omega,\bar\sigma)Y_{ik}^2) = \frac{w_{ik}^2\bar{\sigma}_{jl,jl}}{\bar{\sigma}^{ik}}+\frac{w_{jl}^2\bar{\sigma}_{ik,ik}}{\bar{\sigma}^{jl}}.
 \end{align*}
 The right hand side is bounded because of Condition C0 and Lemma~\ref{lemma1}(iv), and the fact that $e_{ik}(\bar{\omega},\bar{\sigma})$ and $Y_{jl}$ are independent. \\
 For $(i,j,k,l)\in \mathcal{A}$, denote 
 $$
 D:=\bar{L}''_{ijkl,ijkl}(\bar{\omega},\bar{\sigma})-\bar{L}''_{ijkl,\mathcal{A}_{-ijkl}}(\bar{\omega},\bar{\sigma})\Big[\bar{L}''_{\mathcal{A}_{-ijkl},\mathcal{A}_{-ijkl}}(\bar{\omega},\bar{\sigma})\Big]^{-1}\bar{L}''_{\mathcal{A}_{-ijkl},ijkl}(\bar{\omega},\bar{\sigma}).
 $$
 Then $D^{-1}$ is the $(ijkl,ijkl)$th entry in $\Big[\bar{L}''_{\mathcal{A},\mathcal{A}}(\bar{\omega},\bar{\sigma})\Big]^{-1}$. Thus by Lemma~\ref{lemma1}(v), $D^{-1}$ is positive and bounded from above, so $D$ is bounded away from zero. This proves (iii). \\
  Note that $\lVert\bar{L}''_{ijkl,\mathcal{A}}(\bar{\omega},\bar{\sigma})[\bar{L}''_{\mathcal{A},\mathcal{A}}(\bar{\omega},\bar{\sigma})]^{-1}\rVert_2^2\leq\lVert\bar{L}''_{ijkl,\mathcal{A}}(\bar{\omega},\bar{\sigma})\rVert_2^2\lambda_{\max}([\bar{L}''_{\mathcal{A},\mathcal{A}}(\bar{\omega},\bar{\sigma})]^{-2}).$ 
 By Lemma~\ref{lemma1}(iv),\\ $\lambda_{\max}([\bar{L}''_{\mathcal{A},\mathcal{A}}(\bar{\omega},\bar{\sigma})]^{-2})$ is bounded from above, thus it suffices to show that $\lVert\bar{L}''_{ijkl,\mathcal{A}}(\bar{\omega},\bar{\sigma})\rVert_2^2$ is bounded. Define $\mathcal{A}^+:=(i,j,k,l)\cup \mathcal{A}$. Then $\bar{L}''_{ijkl,ijkl}(\bar{\omega},\bar{\sigma})-\bar{L}''_{ijkl,\mathcal{A}}(\bar{\omega},\bar{\sigma})[\bar{L}''_{\mathcal{A},\mathcal{A}}(\bar{\omega},\bar{\sigma})]^{-1}\bar{L}''_{\mathcal{A},ijkl}(\bar{\omega},\bar{\sigma})$ is the inverse of the $(kl,kl)$ entry of $\bar{L}''_{\mathcal{A}^+,\mathcal{A}^+}(\bar{\omega},\bar{\sigma})$. Thus by Lemma~\ref{lemma1}(iv), it is bounded away from zero. Therefore by Lemma~\ref{lemma2}(i), $\bar{L}''_{ijkl,\mathcal{A}}(\bar{\omega},\bar{\sigma})[\bar{L}''_{\mathcal{A},\mathcal{A}}(\bar{\omega},\bar{\sigma})]^{-1}\bar{L}''_{\mathcal{A},ijk}(\bar{\omega},\bar{\sigma})$ is bounded from above. Since $$\bar{L}''_{ijkl,\mathcal{A}}(\bar{\omega},\bar{\sigma})[\bar{L}''_{\mathcal{A},\mathcal{A}}(\bar{\omega},\bar{\sigma})]^{-1}\bar{L}''_{\mathcal{A},ijk}(\bar{\omega},\bar{\sigma})\geq\lVert\bar{L}''_{ijkl,\mathcal{A}}(\bar{\omega},\bar{\sigma})\rVert_2^2\lambda_{\min}([\bar{L}''_{\mathcal{A},\mathcal{A}}(\bar{\omega},\bar{\sigma})]^{-1}),$$
 and by Lemma~\ref{lemma1}(iv), $\lambda_{\min}([\bar{L}''_{\mathcal{A},\mathcal{A}}(\bar{\omega},\bar{\sigma})]^{-1})$ is bounded away from zero, we have $\lVert\bar{L}''_{ijkl,\mathcal{A}}(\bar{\omega},\bar{\sigma})\rVert_2^2$ bounded from above. Thus (iv) follows. 
 \end{proof}

\begin{proof}[of Lemma~\ref{lemma3}]
\rm
The $(i'k', j'l')$th entry of the matrix $Y_{ik}Y_{jl}\tilde{Y}\tilde{Y}^T$ is $Y_{ik}Y_{jl}\tilde{Y}_{i'k'}\tilde{Y}_{j'l'}$, for $1\leq i <j\leq p$, $1\leq k'\leq K_{i'}$ and $1\leq l'\leq K_{j'}$. Hence, the $(i'k', j',l')$th entry of the matrix $\text{E}[Y_{ik}Y_{jl}\tilde{Y}\tilde{Y}^T]$ is $\text{E}[Y_{ik}Y_{jl}\tilde{Y}_{i'k'}\tilde{Y}_{j'l'}]=(\bar{\sigma}_{ik,jl}\bar{\sigma}_{i'k',j'l'}+\bar{\sigma}_{ik,i'k'}\bar{\sigma}_{jl,j'l'}+\bar{\sigma}_{ik,j'l'}\bar{\sigma}_{jl,i'k'})/(\bar{\sigma}^{i'k'}\bar{\sigma}^{j'l'})$, where $\bar{\sigma}_{ik,jl}$ denotes the covariance between $Y_{ik}$ and $Y_{jl}$. Thus, we can write
\begin{equation}
\text{E}[Y_{ik}Y_{jl}\tilde{Y}\tilde{Y}^T]=\frac{1}{\bar{\sigma}^{i'k'}\bar{\sigma}^{j'l'}}(\bar{\sigma}_{ik,jl}\bar{\Sigma}+\bar{\sigma}_{ik,\cdot}\bar{\sigma}_{jl,\cdot}^T+\bar{\sigma}_{jl,\cdot}\bar{\sigma}_{ik,\cdot}^T),
\end{equation}
where $\bar{\sigma}_{ik,\cdot}$ is the $\sum_{j=1}^pK_j$ vector $(\bar{\sigma}_{ik,jl}: l=1,\ldots,K_j, j=1,\ldots, p, j\ne i)$. Then we have
\begin{equation}
\Vert\text{E}[Y_{ik}Y_{jl}\tilde{Y}\tilde{Y}^T]\rVert\leq\frac{1}{\vert\bar{\sigma}^{i'k'}\bar{\sigma}^{j'l'}\rvert}(\lvert\bar{\sigma}_{ik,jl}\rvert\lVert\bar{\Sigma}\rVert+2\lVert\bar{\sigma}_{ik,\cdot}\rVert_2\lVert\bar{\sigma}_{jl,\cdot}\rVert_2),
\end{equation}
where $\lVert\cdot\rVert$ is the operator norm. By Condition C1, $\vert\bar{\sigma}^{i'k'}\bar{\sigma}^{j'l'}\rvert^{-1}$ and $\lvert\bar{\sigma}_{ik,jl}\rvert\lVert\bar{\Sigma}\rVert$ are uniformly bounded. Further  
$\bar{\sigma}_{ik,ik}-\bar{\sigma}_{ik,\cdot}^T\bar{\Sigma}^{-1}_{(-ik)}\bar{\sigma}_{ik,\cdot}>0$, 
where $\bar{\Sigma}_{(-ik)}$ is the submatrix of $\bar{\Sigma}$ removing $ik$th row and column. From this, it follows that 
\begin{equation}
\lVert\bar{\sigma}_{ik,\cdot}\rVert_2  = \lVert\bar{\Sigma}^{1/2}_{-(ik)}\bar{\Sigma}^{-1/2}_{-(ik)}\bar{\sigma}_{ik,\cdot}\rVert_2
 \leq \lVert\bar{\Sigma}^{1/2}_{-(ik)}\rVert\lVert\bar{\Sigma}^{- 1/2}_{-(ik)}\bar{\sigma}_{ik,\cdot}\rVert \leq \sqrt{\lVert\bar{\Sigma}\rVert}\sqrt{\bar{\sigma}_{ik,ik}},
\end{equation}
which follows from the fact that $\bar{\Sigma}_{(-ik)}$ is a principal submatrix of $\bar{\Sigma}$.  
\end{proof}

\begin{proof}[of Lemma~\ref{lemma4}]
\rm
Observe that $L'_{n,ijkl}(\bar{\omega},\sigma,\bm{Y})$ is given by 
$$\frac1n\sum_{m=1}^nw_{ik}\bigg(Y_{ik}^m+\sum_{j'\neq i}\sum_{l'=1}^{K_{j'}}\frac{\omega_{ij'k'l'}}{\sigma^{ik}}Y_{j'l'}^m\bigg)\frac{Y_{jl}^m}{\sigma^{ik}}+w_{jl}\bigg(Y_{jl}^m+\sum_{i'\neq j}\sum_{k'=1}^{K_{i'}}\frac{\omega_{ij'k'l'}}{\sigma^{jl}}Y_{i'k'}^m\bigg)\frac{Y_{ik}^m}{\sigma^{jl}}.$$
Thus $L'_{n,ijkl}(\bar{\omega},\bar{\sigma},\bm{Y})-L'_{n,ijkl}(\bar{\omega},\hat{\sigma},\bm{Y})$ is given by 
\begin{eqnarray*}
\lefteqn{ w_{ik}\Big(\overline{Y_{ik}Y_{jl}}(\frac{1}{\bar{\sigma}^{ik}}-\frac{1}{\hat{\sigma}^{ik}})+\sum_{j'\neq i}\sum_{l'=1}^{K_{j'}}\overline{Y_{j'l'}Y_{jl}}(\frac{1}{(\bar{\sigma}^{ik})^2}-\frac{1}{(\hat{\sigma}^{ik})^2})\Big)}\\
&&+w_{jl}\Big(\overline{Y_{ik}Y_{jl}}(\frac{1}{\bar{\sigma}^{jl}}-\frac{1}{\hat{\sigma}^{jl}})+\sum_{i'\neq j}\sum_{k'=1}^{K_{i'}}\overline{Y_{i'k'}Y_{ik}}(\frac{1}{(\bar{\sigma}^{jl})^2}-\frac{1}{(\hat{\sigma}^{jl})^2})\Big),
\end{eqnarray*}
where $\overline{Y_{ik}Y_{jl}}=\frac1n\sum\limits_{m=1}^nY_{ik}^mY_{jl}^m$. By Lemma~\ref{lemma1}(iv), $\{\bar{\sigma}_{ik,jl}:1\leq i,j\leq p,1\leq k,l\leq K\}$ are bounded from below and above, and hence 
$\max_{i,j,k,l}|\overline{Y_{ik}Y_{jl}}-\bar{\sigma}_{ik,jl}|=O_p(\sqrt{(\log n)/{n}})$. 
This implies that ${\max}_{i,j,k,l}\lvert\overline{Y_{ik}Y_{jl}}\rvert=O_p(1)$, and hence by Lemma~\ref{lemma1}(iv) and Condition C3 it follows that 
$$
\max_{i,j,k,l}|L'_{n,ijkl}(\bar{\omega},\bar{\sigma},\bm{Y})-L'_{n,ijk}(\bar{\omega},\hat{\sigma},\bm{Y})|=O_p \left(\sqrt{\frac{\log n}{n}}\right).
$$
The bound for  $|L''_{n,ijkl,tsk'l'}(\bar{\omega},\bar{\sigma},\bm{Y})-L''_{n,ijkl,tsk'l'}(\bar{\omega},\hat{\sigma},\bm{Y})|$ follows similarly.
\end{proof}

\begin{proof}[of Lemma~\ref{lemma5}]
\rm
If we replace $\hat{\sigma}$ by $\bar{\sigma}$ on the left hand side and take $(i,j,k,l)\in\mathcal{A}$, then 
from the definition we have $L'_{n,ijkl}(\bar{\omega},\bar{\sigma},\bm{Y})=\bm{e}_{ik}(\bar{\omega},\bar{\sigma})^T\bm{Y}_{jl}+\bm{e}_{jl}(\bar{\omega},\bar{\sigma})^T\bm{Y}_{ik}$, and $\bm{Y}_{jl}$, where $\bm{e}_{ik}$ are $n$ replications of $e_{ik}(\bar{\omega,\bar{\sigma}})$. Thus by Lemma~\ref{app2} of the Appendix we obtain 
$
{\max}\{\lvert L'_{n,ijkl}(\bar{\omega},\hat{\sigma},\bm{Y})\rvert: {(i,j,k,l)\in\mathcal{A}}\}\leq C_n\sqrt{({\log n})/{n}}.$
and hence by the Cauchy-Schwartz inequality 
\begin{equation*}
\lVert L'_{n,\mathcal{A}}(\bar{\omega},\hat{\sigma},\bm{Y})\rVert_2\leq K_{\max}\sqrt{q_n}\underset{(i,j,k,l)\in\mathcal{A}}{\max}\lvert L'_{n,ijkl}(\bar{\omega},\hat{\sigma},\bm{Y})\rvert\leq C_nK_{\max}\sqrt{\frac{q_n\log n}{n}}, 
\end{equation*}
and $
\lVert L'_{n,\mathcal{A}}(\bar{\omega},\hat{\sigma},\bm{Y})\rVert_2\leq\lVert L'_{n,\mathcal{A}}(\bar{\omega},\bar{\sigma},\bm{Y})\rVert_2+\lVert L'_{n,\mathcal{A}}(\bar{\omega},\bar{\sigma},\bm{Y})-L'_{n,\mathcal{A}}(\bar{\omega},\hat{\sigma},\bm{Y})\rVert_2
$. 
The second term on the right hand side has order $K_{\max}\sqrt{{q_n(\log n)}/{n}}$. Since there are $K_{\max}^2q_n$ terms and by Lemma~\ref{lemma4}, they are uniformly bounded by $\sqrt{({\log n})/{n}}$. The rest of the lemma can be proved by similar arguments. 
\end{proof}

\begin{proof}[of Lemma~\ref{lemma6}]
\rm
Let $\alpha_n=K_{\max}\sqrt{q_n}\lambda_n$, and $\mathcal{L}_n(\omega,\hat{\sigma},\bm{Y})=L_n(\omega,\hat{\sigma},\bm{Y})+\lambda\sum\sum_{i<j}\lVert\omega_{ij}\rVert_2$. Then for any given constant $\bar{C}_1>0$ and any vector $u$ such that $u_{\mathcal{A}^c}=0$ and $\lVert u\rVert_2=\bar{C}_1$, the triangle inequality and the Cauchy-Schwartz inequality together imply that 
$$
\sum_{i<j}\lVert\bar{\omega}_{ij}\rVert_2-\sum_{i<j}\lVert\bar{\omega}_{ij}+\alpha_nu_{ij}\rVert_2\leq \alpha_n\sqrt{K_{\max}^2q_n}\lVert u\rVert_2=\bar{C}_1\alpha_nK_{\max}\sqrt{q_n}.
$$
Thus $\mathcal{L}_n(\bar{\omega}+\alpha_nu,\hat{\sigma},\bm{Y},\lambda_n)-\mathcal{L}_n(\bar{\omega},\hat{\sigma},\bm{Y},\lambda_n)$ can be written as 
\begin{eqnarray*}
\lefteqn{ \{L_n(\bar{\omega}+\alpha_nu,\hat{\sigma},\bm{Y})-L_n(\bar{\omega},\hat{\sigma},\bm{Y})\}-\lambda_n\{\sum_{i<j}\lVert\bar{\omega}_{ij}\rVert_2-\sum_{i<j}\lVert\bar{\omega}_{ij}+\alpha_nu_{ij}\rVert_2\}}
\\
&&\geq  \{L_n(\bar{\omega}+\alpha_nu,\hat{\sigma},\bm{Y})-L_n(\bar{\omega},\hat{\sigma},\bm{Y})\} - \bar{C}_1\alpha_nK_{\max}\sqrt{q_n}\lambda_n\\
&&=  \{L_n(\bar{\omega}+\alpha_nu,\hat{\sigma},\bm{Y})-L_n(\bar{\omega},\hat{\sigma},\bm{Y})\} -\bar{C}_1\alpha_n^2.
\end{eqnarray*}
Thus for any sequence $C_n\rightarrow\infty$, with probability tending to 1,
\begin{eqnarray*}
\lefteqn{ L_n(\bar{\omega}+\alpha_nu,\hat{\sigma},\bm{Y})-L_n(\bar{\omega},\hat{\sigma},\bm{Y})}\\
&&=  \alpha_nu_\mathcal{A}^TL'_{n,\mathcal{A}}(\bar{\omega},\hat{\sigma},\bm{Y})+\frac12\alpha_n^2u_\mathcal{A}^TL''_{n,\mathcal{A},\mathcal{A}}(\bar{\omega},\hat{\sigma},\bm{Y})u_\mathcal{A}\\
&&=  \frac12\alpha_n^2u_\mathcal{A}^T\bar{L}''_{n,\mathcal{A},\mathcal{A}}(\bar{\omega},\bar{\sigma})u_\mathcal{A} + \frac12\alpha_n^2u_\mathcal{A}^T\Big(L''_{n,\mathcal{A},\mathcal{A}}(\bar{\omega},\hat{\sigma},\bm{Y})-\bar{L}''_{n,\mathcal{A},\mathcal{A}}(\bar{\omega},\bar{\sigma})\Big)u_\mathcal{A}+ \alpha_nu_\mathcal{A}^TL'_{n,\mathcal{A}}(\bar{\omega},\hat{\sigma},\bm{Y}) \\
&&\geq  \frac12\alpha_n^2u_\mathcal{A}^T\bar{L}''_{n,\mathcal{A},\mathcal{A}}(\bar{\omega},\bar{\sigma})u_\mathcal{A} -C_n\alpha_n^2K_{\max}^2q_nn^{-1/2}\sqrt{\log n}- C_n\alpha_nK_{\max}q_n^{1/2}n^{-1/2}\sqrt{\log n}.
\end{eqnarray*}
In the above, the first equation holds because the loss function $L(\omega,\sigma, Y)$ is quadratic in $\omega$ and $u_{\mathcal{A}^c}=0$. The inequality is due to Lemma~\ref{lemma5}.\\
By the assumptions that $K_{\max}^2q_n=o(\sqrt{n/\log n})$ and $\lambda_n\sqrt{n/\log n}\rightarrow\infty$, we have  $\alpha_n^2K_{\max}^2q_nn^{-1/2}\sqrt{\log n}=o(\alpha_n^2)$ and $\alpha_nK_{\max}q_n^{1/2}n^{-1/2}\sqrt{\log n}=o(\alpha_n^2)$. Thus,
$$
\mathcal{L}_n(\bar{\omega}+\alpha_nu,\hat{\sigma},\bm{Y},\lambda_n)-\mathcal{L}_n(\bar{\omega},\hat{\sigma},\bm{Y},\lambda_n) \geq \frac14\alpha_n^2u_\mathcal{A}^T\bar{L}''_{\mathcal{A},\mathcal{A}}(\bar{\omega},\bar{\sigma})u_\mathcal{A}-\bar{C}_1\alpha_n^2
$$
with probability tending to 1.
By Lemma~\ref{lemma1} (iv),  $u_\mathcal{A}^T\bar{L}''_{\mathcal{A},\mathcal{A}}u_\mathcal{A}\geq\Lambda_{\min}^L(\bar{\omega},\bar{\sigma})\lVert u_\mathcal{A}\rVert_2^2=\Lambda_{\min}^L(\bar{\omega},\bar{\sigma})\bar{C}_1^2$, thus if we take $\bar{C}_1=5/\Lambda_{\min}^L(\bar{\omega},\bar{\sigma})$, then 
$$
\mathrm{P}\left[ \inf \{\mathcal{L}_n(\bar{\omega}+\alpha_nu,\hat{\sigma},\bm{Y},\lambda_n): {u:u_{\mathcal{A}^c}=0,\lVert u\rVert_2=\bar{C}_1}\}>\mathcal{L}_n(\bar{\omega},\hat{\sigma},\bm{Y},\lambda_n)\right ]\to 1.
$$
Hence a local minimum exists in $\{\omega:\lVert\omega-\hat{\omega}\rVert_2\leq \bar{C}_1K_{\max}\sqrt{q_n}\lambda_n\}$ with probability tending to 1. 
\end{proof}

\begin{proof}[of Lemma~\ref{lemma7}]
\rm
Let $\alpha_n=K_{\max}\sqrt{q_n}\lambda_n$. Any $\omega$ in the statement of the lemma can be written as $\omega=\bar{\omega}+\alpha_nu$, with $u_{\mathcal{A}^c}=0$ and $\lVert u\rVert_2\geq \bar{C}_2$, where $\bar{C}_2>0$. Note that
\begin{align*}
L'_{n,\mathcal{A}}(\omega, \hat{\sigma},\bm{Y}) & = L'_{n,\mathcal{A}}(\bar{\omega},\hat{\sigma},\bm{Y})+\alpha_nL''_{n,\mathcal{A},\mathcal{A}}(\bar{\omega},\hat{\sigma},\bm{Y})u\\
& = L'_{n,\mathcal{A}}(\bar{\omega},\hat{\sigma},\bm{Y}) + \alpha_n\Big(L''_{n,\mathcal{A},\mathcal{A}}(\bar{\omega},\hat{\sigma},\bm{Y})-\bar{L}''_{\mathcal{A},\mathcal{A}}(\bar{\omega},\bar{\sigma})\Big)u+\alpha_n\bar{L}''_{\mathcal{A},\mathcal{A}}(\bar{\omega},\bar{\sigma})u.
\end{align*} 
By the triangle inequality and Lemma~\ref{lemma5}, for any  $C_n\rightarrow\infty$, $\lVert L'_{n,\mathcal{A}}(\omega,\hat{\sigma},\bm{Y})\rVert_2$ is bounded below by 
$$\alpha_n\lVert \bar{L}''_{\mathcal{A},\mathcal{A}}(\bar{\omega},\bar{\sigma})u\rVert_2-C_n(K_{\max}q_n^{1/2}n^{-1/2}\sqrt{\log n})-C_n\lVert u\rVert_2(\alpha_nK_{\max}^2q_nn^{-1/2}\sqrt{\log n})
$$
with probability tending to 1. Thus, as argued in the proof of Lemma~\ref{lemma6}, $\alpha_nK_{\max}q_n^{1/2}n^{-1/2}\sqrt{\log n}=o(\alpha_n)$ and $\alpha_nK_{\max}^2q_nn^{-1/2}\sqrt{\log n}=o(\alpha_n)$, then $\lVert L'_{n,\mathcal{A}}(\omega,\hat{\sigma},\bm{Y})\rVert_2\geq \frac12\alpha_n\lVert\bar{L}''_{\mathcal{A},\mathcal{A}}(\bar{\omega},\bar{\sigma})u\rVert_2$ with probability tending to 1. By Lemma~\ref{lemma1}(iv), $\lVert\bar{L}''_{\mathcal{A},\mathcal{A}}(\bar{\omega},\bar{\sigma})u\rVert_2\geq\Lambda^L_{\min}(\bar{\omega},\bar{\sigma})\lVert u\rVert_2$. Therefore $\bar{C}_2$ can be taken as $3/\Lambda^L_{\min}(\bar{\omega},\bar{\sigma})$.
\end{proof}

\begin{proof}[of Lemma~\ref{lemma8}]
\rm
Observe that
$\text{Var}(D_{\mathcal{A},ijkl}(\bar{\omega},\bar{\sigma},Y))=\text{E}(L''_{\mathcal{A},ijkl}(\bar{\omega},\bar{\sigma},Y)L''_{\mathcal{A},ijkl}(\bar{\omega},\bar{\sigma},Y)^T)-\bar{L}''_{\mathcal{A},ijkl}(\bar{\omega},\bar{\sigma})\bar{L}''_{\mathcal{A},ijkl}(\bar{\omega},\bar{\sigma})^T.$ 
Thus it suffices to show that there exists a constant $M_5>0$, such that for all $(i,j,k,l)$,
$
\lambda_{\max}(\text{E}(L''_{\mathcal{A},ijkl}(\bar{\omega},\bar{\sigma},Y)L''_{\mathcal{A},ijkl}(\bar{\omega},\bar{\sigma},Y)^T))\leq M_5.
$
We use the same notations as in the proof of Lemma~\ref{lemma1}(v).\\
Note that $L''_{\mathcal{A},ijkl}(\bar{\omega},\bar{\sigma},Y)=\tilde{X}^T\tilde{X}_{(ik,jl)}=Y_{ik}X_{jl}+Y_{jl}X_{ik}$. Thus $\text{E}(L''_{\mathcal{A},ijkl}(\bar{\omega},\bar{\sigma},Y)L''_{\mathcal{A},ijkl}(\bar{\omega},\bar{\sigma},Y)^T)$ is given by 
$\text{E}[Y_{ik}^2X_{jl}X_{jl}^T]+\text{E}[Y_{jl}^2X_{ik}X_{ik}^T]+\text{E}[Y_{ik}Y_{jl}(X_{jl}X_{jl}^T+X_{ik}X_{ik}^T)],$ 
and for $a\in \mathbb{R}^d$,\\
$ a^T\text{E}_{\bar{\omega},\bar{\sigma}}(L''_{\mathcal{A},ijkl}(\bar{\omega},\bar{\sigma},Y)L''_{\mathcal{A},ijkl}(\bar{\omega},\bar{\sigma},Y)^Ta
=  a_{jl}^T\text{E}[Y_{ik}^2\tilde{Y}\tilde{Y}^T]a_{jl}+a_{ik}^T\text{E}[Y_{jl}^2\tilde{Y}\tilde{Y}^T]a_{ik}+2a_{ik}^T\text{E}[Y_{ik}Y_{jl}\tilde{Y}\tilde{Y}^T]a_{jl}$. 
Since $\sum_{i=1}^p\sum_{k=1}^{K_i}\lVert a_{ik}\rVert_2^2=2\lVert a\rVert_2^2=2$, and by Lemma~\ref{lemma3}  $\lambda_{\max}(\text{E}[Y_{ik}Y_{jl}\tilde{Y}\tilde{Y}^T])\leq M_4$ for any $1\leq i<j\leq p$ and $1\leq k\leq K_i$, $1\leq l\leq K_j$, the conclusion follows. 
\end{proof}

\section*{Appendix B. Auxiliary results}
\label{sec:appendix}

\begin{lemma} 
\label{app1}
 Let $X_{ij}\sim\text{N}(0,\sigma_i^2)$, $i=1,\dots,m$ and $j=1,\dots, n$. For each $i$, $X_{i1},\dots,X_{in}$ are assumed to be i.i.d., but are arbitrarily dependent across $i$. Then for any sequence $C_n\rightarrow\infty$, with probability tending to $1$, we have
${\max}_{1\le i\le m}\lvert{n}^{-1}\sum_{j=1}^nX_{ij}^2-\sigma_{i}^2\rvert\leq C_n\sqrt{({\log m})/{n}}.$ 
\end{lemma}

\begin{proof}
\rm
Let $Z_{ij}={X_{ij}}/{\sigma_i}$, then for fixed $i$ and $r=2,3,\dots$, we have
\begin{equation*}
\text{E}\lvert{n}^{-1}(Z_{i1}^2-1)\rvert^r\leq \frac{2^{r-1}}{n^r}\text{E}(Z_{i1}^{2r}+1)\leq ({2}/{n})^rr!=({2}/{n})^{r-2}\frac{4}{n^2}r!.
\end{equation*}
By  Lemma 2.2.11 of Van Der Vaart \& Wellner \cite{empirical}, taking $M={2}/{n}$ and $v={8}/{n}$, it follows that 
$\text{P}\Big(\lvert{n}^{-1}\sum_{j=1}^nZ_{ij}^2-1\rvert>x\Big)\leq 2e^{-{x^2}/[2(8/n+2x/n)]}$.  
Since $\sigma_i$ are bounded, Lemma 2.2.10 of Van Der Vaart \& Wellner \cite{empirical} implies that for some $C>0$,
$\text{E}\Big({\max}_{1\leq i\leq m}\vert n^{-1}\sum_{j=1}^n X_{ij}^2-\sigma_i^2\vert\Big)\leq C\sqrt{({\log m})/{n}}$,
which implies the conclusion.
\end{proof}

\begin{lemma}
\label{app2}
 Let $X_{ij}\overset{i.i.d.}{\sim}\text{N}(0,\sigma_{xi}^2)$ and $Y_{ij}\overset{i.i.d.}{\sim}\text{N}(0,\sigma_{yi}^2)$ for $i=1,\dots,m$ and $j=1,\dots,n$, and $X_{ij}$ and $Y_{ij}$ are independent for all $i$. Further assume that $0<\sigma_{xi},\sigma_{yi}\leq\sigma<\infty$. Then for any sequence $C_n\rightarrow\infty$, we have 
${\max}_{1\leq i\leq m}\lvert{n}^{-1}\sum_{j=1}^n X_{ij}Y_{ij}\rvert \leq C_n\sqrt{({\log m})/{n}}.$ 
\end{lemma}

\begin{proof}
\rm
For fixed $i$ we can observe that
\begin{equation*}
\text{E}\lvert{n}^{-1}X_{i1}Y_{i1}\rvert^r = \frac{1}{n^r}\text{E}\lvert X_{i1}\rvert^r\text{E}\lvert Y_{i1}\rvert^r\leq \frac{2^r\sigma^r}{n^r}\frac{(\Gamma(\frac{r+1}{2}))^2}{\pi}
\leq ({2\sigma}/{n})^{r-2}\frac{4\sigma^2}{\pi n^2}r!.
\end{equation*}
By  Lemma 2.2.11 of Van Der Vaart \& Wellner \cite{empirical}, taking $M={2\sigma}/{n}$ and $v={8\sigma^2}/{\pi n}$, we have
$
\text{P}\Big(\lvert{n}^{-1}\sum_{j=1}^nX_{ij}Y_{ij}\rvert>x\Big)\leq 2e^{-{x^2}/[2({8\sigma^2/\pi n+2\sigma x/n)]}}.
$
Then by Lemma 2.2.10 of Van Der Vaart \& Wellner \cite{empirical}, for some $C>0$, 
$\text{E}\Big(\underset{1\leq i\leq m}{\max}\lvert{n}^{-1}\sum_{j=1}^nX_{ij}Y_{ij}\rvert \Big)\leq C\sqrt{({\log m})/{n}}$, 
which implies the conclusion.
\end{proof}

\bibliography{ref}
\bibliographystyle{plain}

\end{document}